\documentclass{article} 
\usepackage{iclr2025_conference,times}

\usepackage{amsmath,amsfonts,bm}

\def\eqref#1{equation~\ref{#1}}

\def\1{\bm{1}}

\DeclareMathAlphabet{\mathsfit}{\encodingdefault}{\sfdefault}{m}{sl}
\SetMathAlphabet{\mathsfit}{bold}{\encodingdefault}{\sfdefault}{bx}{n}

\usepackage{hyperref}
\usepackage{url}
\usepackage[american]{babel}
\usepackage{amsmath}
\usepackage{amssymb}
\usepackage{amsthm}
\usepackage{mathtools} 
\usepackage{booktabs} 
\usepackage{tikz} 
\usepackage{dsfont} 
\usepackage{booktabs}
\usepackage{multirow}
\usepackage{comment}
\usepackage{xspace}
\usepackage{listings}
\usepackage{makecell}
\newtheorem{theorem}{Theorem}[section]
\newtheorem{proposition}[theorem]{Proposition}

\newtheorem{definition}[theorem]{Definition}

\usepackage{siunitx}
\newcommand{\mesure}[2]{%
  #1{\scriptsize $\pm$#2}%
}
\newcommand{\mesurebf}[2]{%
  \textbf{#1}{\scriptsize $\pm$#2}%
}
\newcommand{\float}[2]{%
  #1{\scriptsize $\times 10^{#2}$}%
}
\newcommand{\floatbf}[2]{%
  \textbf{#1}{\scriptsize \boldmath $\times 10^{#2}$}%
}
\usepackage[table]{xcolor} 
\usepackage{booktabs}      
\usepackage{array}         
\usepackage{tabularx}      

\definecolor{improveGreen}{RGB}{144,238,144} 
\definecolor{green}{RGB}{0,128,0}            
\definecolor{degradeRed}{RGB}{255,182,193}   
\definecolor{red}{RGB}{220,20,60}            

\newcommand{\cellheatmap}[3]{%
\begin{tikzpicture}[baseline=(gridcenter.base), scale=0.4]
  \foreach \i in {0,1,2,3}{
    \foreach \j in {0,1,2}{
      \pgfmathparse{int(#1[\i*3+\j])}
      \edef\val{\pgfmathresult}

      \ifnum\val=2 \def\colorcell{green} \fi
      \ifnum\val=1 \def\colorcell{improveGreen} \fi
      \ifnum\val=0 \def\colorcell{white} \fi
      \ifnum\val=-1 \def\colorcell{degradeRed} \fi
      \ifnum\val=-2 \def\colorcell{red} \fi

      \fill[\colorcell] (\j,-\i) rectangle ++(1,-1);
      \draw[black, line width=0.05pt] (\j,-\i) rectangle ++(1,-1);
    }
  }

  \path (0,0.5) rectangle (3,-3.5);

  \node (gridcenter) at (1.5,-1.5) {};
  \node[left, align=right, inner sep=1pt, font=\normalsize]
    at (-0.2,-1.5) {#2\\[-1pt]{\footnotesize #3}};
\end{tikzpicture}%
}

\newcommand{\cellheatmapOneByThreeRaw}[1]{%
\begin{tikzpicture}[baseline=(gridcenter.base)]
  \foreach \i in {0,1,2}{
      \pgfmathparse{int(#1[\i])}
      \edef\val{\pgfmathresult}

      \ifnum\val=2 \def\colorcell{green} \fi
      \ifnum\val=1 \def\colorcell{improveGreen} \fi
      \ifnum\val=0 \def\colorcell{white} \fi
      \ifnum\val=-1 \def\colorcell{degradeRed} \fi
      \ifnum\val=-2 \def\colorcell{red} \fi

      \fill[\colorcell] (0,-\i) rectangle ++(1,-1);
      \draw[black, line width=0.1pt] (0,-\i) rectangle ++(1,-1);
  }

  \node (gridcenter) at (0.5,-1.5) {};
\end{tikzpicture}%
}

\newsavebox{\valuebox}
\newlength{\halfheight}

\newcommand{\indicatrice}[1]{\mathds{1}\{#1\}}

\newcommand{\gpt}{\texttt{gpt-oss-20b\xspace}}
\newcommand{\closedgpt}{\texttt{GPT-5.6 Luna\xspace}}
\newcommand{\deepseek}{\texttt{DeepSeek-R1-32B\xspace}}
\newcommand{\qwen}{\texttt{Qwen3-4B\xspace}}

\usepackage{todonotes}

\title{Jailbreaks for Black-Box Uncertainty Quantification in Large Reasoning Models}

\iclrfinalcopy
\author{
Lucas Biéchy\thanks{Corresponding author: \texttt{lucas.biechy@inria.fr}}$^{~~,1,2,3}$ \quad
Cédric Eichler$^{1,3,4,5}$ \quad
Adrien Boiret$^{1,3,4,5}$ \quad
Nicolas Anciaux$^{1,2,3}$ \\
$^1$Petscraft, Inria \quad
$^2$Université Paris-Saclay \quad
$^3$INSA CVL \quad
$^4$Université d'Orléans  \quad
$^5$LIFO
}

\begin{document}
\maketitle

\begin{abstract} 
While Large Reasoning Models (LRMs) excel at complex reasoning, alignment through reinforcement learning often induces systemic overconfidence. In production environments, where logits may be unavailable, robust black-box uncertainty quantification (UQ) is essential for trustworthiness and safety. Focusing on question-answering for LRMs, we show that existing black-box methods, such as paraphrase-based self-consistency and confidence verbalization, offer little to no improvement over simple repeated sampling, suggesting that alignment suppresses useful output variability. We introduce prompt-level relaxation operators that broaden the model's effective output distribution by approximating the effect of an optimal policy obtained with a stronger KL-regularization parameter, hence closer to the reference model. Theoretically, we demonstrate that relaxation improves calibration. We propose Jailbreak for Uncertainty (J4U), a jailbreak-derived technique for UQ that empirically reproduces the behavioral signatures predicted by our relaxation theory. Across 3 datasets and 4 LRMs, including a closed-source production model, J4U's improvement over repeated sampling achieves statistical significance in up to 6$\times$ more LRM–dataset-metric settings  than the strongest black-box UQ state-of-the-art baseline we evaluate, with average ECE reductions up to 5$\times$ larger. These results provide a practical tool for UQ in black-box LRM deployment.
\end{abstract}

\section{Introduction}\label{sec:intro}

The emergence of Large Reasoning Models (LRMs) has marked a significant paradigm shift from traditional Large Language Models (LLMs) \citep{XU2025101370}. Designed to explicitly develop intermediate reasoning before producing an answer, LRMs achieve unprecedented levels of accuracy across a wide range of tasks \citep{deepseek,openai2024gpt4technicalreport}. As a result, they are increasingly deployed as direct decision-makers, even in high-stakes domains, often framed as question-answering (QA) tasks~\citep{10.1093/jamiaopen/ooae154,doi:10.1056/AIp2300031,10.1007/978-3-032-01399-6_5}. In such settings, it is critical that the model’s confidence in its own decisions be rigorously calibrated to its true probability of being correct \citep{shorinwa2025survey}. 

However, like their predecessors, LRMs are often accessible as black-boxes \citep{wan20252025foundationmodeltransparency}, which limits applicable standard uncertainty quantification (UQ) approaches to prompt-level methods~\citep{shorinwa2025survey} such as confidence verbalization (e.g.~\citet{lin2022teaching,xiong2023can,yang2024verbalizedconfidencescoresllms}) and perturbation-based consistency (e.g.~\citet{portillo-wightman-etal-2023-strength,JiangCalibratingLM,pedapati2024large,yang2024just}). To date, the direct impact of the transition from LLMs to LRMs on confidence calibration in such settings remains poorly understood. It is therefore essential to investigate this configuration and evaluate the effectiveness of existing UQ methods within this new paradigm.

When evaluating state-of-the-art black-box UQ methods originally developed for LLMs on LRMs, we find that they are in fact comparable to a naive self-consistency baseline based on repeated sampling. In practice, LRMs inherit the well-documented overconfidence issues of LLMs \citep{epstein2025llmsoverconfidentevaluatingconfidence,groot-valdenegro-toro-2024-overconfidence}, resulting in limited predictive variability. This phenomenon is closely tied to the alignment phase, which aims to fulfill a twofold objective of maximizing complex reasoning capabilities, and preventing the generation of harmful or dangerous content~\citep{openai2025gptoss120bgptoss20bmodel,deepseek}. To meet this objective, models are extensively optimized using Reinforcement Learning (RL) \citep{schulman2017proximal}, and LRMs in particular rely heavily on RL to forge their reasoning behavior. Prior works suggest that such optimization can mechanically sharpen output probability distributions and favor more deterministic reasoning paths~\citep{openai2024gpt4technicalreport,ICLR2025_29fb6e14}. Moreover, our empirical observation shows that existing calibration methods often become indistinguishable from simple repeated sampling.

To overcome this limitation, we introduce a prompt-level relaxation operator. This input transformation acts as a restoring force toward the model’s pre-alignment base policy, prior to its RL-based alignment. From a theoretical perspective, we model the aligned model as a closed-form optimal policy \citep{rafailov2023direct}. Under this view, relaxation corresponds to approximating a policy trained with a higher KL-regularization parameter. We show that relaxation improves confidence  evaluation, i.e., it lowers Expected Calibration Error (ECE). 

Jailbreak methods were originally designed to bypass the safety side of alignment~\citep{10.5555/3666122.3669630,10.5555/3692070.3694246}. Intuitively, since these methods act against alignment-induced constraints, they are natural candidates to restore output variability that is suppressed by RL-based alignment. Building on this intuition, we propose \textsc{J4U} (Jailbreak for Uncertainty), a black-box self-consistency uncertainty estimator that repurposes jailbreak-style prompt transformations for UQ on strictly benign questions while sampling multiple traces. Leveraging the taxonomy introduced by \citet{shen2025safetyinstinctsllmslearn}, we instantiate one representative transformation from three categories (\textsc{J4U-Suffix}, \textsc{J4U-Prog}, and \textsc{J4U-Art}).

\noindent\textbf{Contribution.} Our contributions, in the context of black-box UQ, are as follows:
\begin{enumerate}
    \item We introduce \textbf{relaxation operators} for black-box LRMs and provide a theoretical calibration guarantee that relaxation improves calibration (lower ECE).
    \item We propose \textbf{\textsc{J4U}}, jailbreak-derived perturbations for UQ. Averaged over 4 LRMs and 3 datasets, the best J4U variant achieves 15.8\%, 14.5\%, and 4.8\% reduction of ECE, NLL, and Brier  over repeated stochastic sampling, vs 3.1\%, 2\%, and \textit{a $2\%$ increase} for the strongest evaluated state of the art approach.
    \item We show that \textbf{\textsc{J4U}} operators reproduce the observable signatures predicted by relaxation, namely flattening output distributions and diversification reasoning traces.

    \end{enumerate}

The remainder of the paper is organized as follows. Section~\ref{sec:rw} reviews related work and positions our proposal. Section~\ref{sec:background} introduces the problem statement and model. Section~\ref{sec:J4U} describes relaxation operators and their theoretical foundations. Section~\ref{sec:xp} introduces and empirically validates \textsc{J4U}. Section~\ref{sec:discu} provides concluding remarks and discusses limitations.

\section{Related Work}\label{sec:rw}

\noindent\textbf{Black-box uncertainty quantification for QA.} UQ methods aim to estimate how likely a model’s answer is correct~\cite{shorinwa2025survey}. Uncertainty is often decomposed into \textit{epistemic} (lack of knowledge) and \textit{aleatoric} uncertainty (inherent variability)~\citep{shorinwa2025survey}. In black-box QA, where only input/output texts access is available and outputs are often limited to a single clause, UQ methods mainly rely on \emph{prompt-level} procedures.

A natural \emph{baseline} is repeated sampling under stochastic decoding, using answer variability as a proxy for uncertainty~\citep{pmlr-v286-xiao25a,pedapati2024large,tao2025revisiting}. A  difficulty with this approach is that agreement between free-form responses is not trivial to assess. \textit{Semantic consistency} methods address this by measuring agreement at the level of meaning~\citep{kuhn2023semantic,manakul-etal-2023-selfcheckgpt,zhao2026seseblackboxuncertaintyquantification}. In the multiple-choice and exact-match settings we study, semantic clustering functionally reduces to the exact-match frequency of the final extracted answer. J4U is complementary to them, as it modifies the distribution from which answers are sampled rather than how agreement is measured. 

Beyond repeated sampling and its semantic extensions, prompt-level black-box UQ techniques for QA mainly fall into two families~\citep{shorinwa2025survey}. \emph{Perturbation-based consistency} approaches couple answer variability analysis with controlled input perturbations~\citep{portillo-wightman-etal-2023-strength,JiangCalibratingLM,pedapati2024large} (e.g. paraphrasing~\citep{yang2024just}), though this may change the effective task distribution and affect accuracy~\citep{WANG20251612}. \emph{Confidence verbalization} approaches prompt models  to explicitly report a confidence score~\citep{lin2022teaching,xiong2023can,yang2024verbalizedconfidencescoresllms}, often combined with multi-sample aggregation to reduce variance~\citep{xiong2023can}. These methods were developed and validated primarily on LLMs. On LRMs, whose predictive variability is reduced, we find they are rarely statistically distinguishable from repeated sampling in accuracy or calibration (Section~\ref{sec:xp}), motivating prompt-level mechanisms that restore usable uncertainty signals.

A distinct line of work builds on conformal prediction (CP), a distribution-free framework that converts any nonconformity or confidence score into prediction sets with finite-sample coverage guarantees~\citep{10.1561/2200000101}. Since CP is agnostic to how its input score is produced, our contribution is complementary rather than competing.

\noindent\textbf{Jailbreak prompting.}
Jailbreak attacks are prompts transformations designed to bypass safety behavior induced by alignment, including RL-based procedure~\citep{10.5555/3666122.3669630,10.5555/3692070.3694246}. Recent taxonomies on \emph{one-shot} black-box jailbreaks~\citep{shen2025safetyinstinctsllmslearn,yi2024jailbreakattacksdefenseslarge} identify a small number of categories, including (i) \emph{suffix/prefix injection} strategies, (ii) \emph{template-based} attacks (e.g.~\citet{10579515,li-etal-2024-drattack}), and (iii) \emph{format manipulations} that preserve the original intent of the prompt while altering its linguistic form (e.g.~\citet{andriushchenko2025does,DBLP:conf/naacl/DingKMCXCH24,jiang-etal-2024-artprompt}). These methods are typically evaluated for safety bypass on harmful prompts, rather than for their effect on predictive uncertainty under benign QA. A small body of work studies how adversarial prompting can \emph{manipulate apparent uncertainty} triggering over- or under-confidence~\citep{zeng2024uncertainty,obadinma2024calibration,obadinma2025robustness}. 

In contrast, our work uses jailbreak-style transformations on \emph{strictly benign} questions to counteract RL-induced sharpening and recover suppressed uncertainty signals for black-box UQ. To the best of our knowledge, our proposal is the first method that leverages jailbreaks to improve UQ.

\section{Problem Formulation and Preliminaries}\label{sec:background}

We consider a multiple choice QA task where $(X, Y) \in \mathcal{X} \times \mathcal{Y}$ is a pair of random variables following a joint distribution $\mathcal{D}$. The input space $\mathcal{X}$ represents possible questions, and $\mathcal{Y}$ a finite set of answers. Let $\mathcal{V}$ be a finite vocabulary of tokens, and $\mathcal{T} \subseteq \mathcal{V}^*$ the reasoning space of finite-length intermediate token sequences . For LRMs, we introduce a latent random variable $T \in \mathcal{T}$ representing the model's reasoning. We assume the existence of a deterministic mapping function $f: \mathcal{T} \rightarrow \mathcal{Y}$ such that for any realized reasoning $t \in \mathcal{T}$, the final predicted answer is uniquely identified as $f(t)$. We treat the model as a \emph{black box}: at inference time we may only submit prompts and access textual answers with no access to reasoning traces, logits, hyperparameters or hidden states. The key question is how to accurately quantify the confidence of the predicted answer $f(t)$ from this sampling interface alone, especially under RL-based alignment which often biases models toward overconfident predictions.

 In RL, reasoning traces are characterized by policy-induced probability distributions. We denote $\pi_T(\cdot \mid x)$ a conditional policy on the reasoning space $\mathcal{T}$ inducing a marginal distribution on $y \in \mathcal{Y}$:
$$
\pi_Y(y \mid x) = \sum_{t \in \mathcal{T}} \indicatrice{f(t)=y}~\pi_T(t \mid x)
$$
For simplicity's sake, we use $\pi$ to denote both $\pi_T$ and $\pi_Y$, arguments resolving any ambiguity.

\begin{definition}[Theoretical Mode]\label{def:theoretical_mode}
Given a policy $\pi$ and its induced marginal distribution over answers, we define the theoretical mode as the most probable answer. Formally,
$$
y^*_\pi(x) \coloneqq \arg\max_{y \in \mathcal{Y}} \pi(y \mid x)
$$
\end{definition}

To study how alignment affects confidence and calibration, we distinguish between a reference policy $\pi_{\text{ref}}$ (typically the pre-trained model) and an aligned policy $\pi_{\text{RL}}$ obtained through RL \citep{ouyang2022training}.  
RL is classically implemented via policy-gradient methods, most notably PPO \citep{schulman2017proximal} and more recently DPO \citep{rafailov2023direct}, GRPO \citep{shao2024deepseekmath}, and GSPO \citep{zheng2025group}.
These methods admit a common closed-form optimal policy \citep{rafailov2023direct, azar2024general, gx-chen2026klregularized}, which can be expressed as
\begin{equation}\label{eq:gibbs_policy}
\pi_\beta(t \mid x) = \frac{1}{Z_\beta(x)} \pi_{\text{ref}}(t \mid x) \exp\left( \frac{R(x, t)}{\beta} \right)
\end{equation}
where $\beta > 0$ is the KL-regularization parameter, $Z_\beta(x) = \sum_{t \in \mathcal{T}} \pi_{\text{ref}}(t|x) \exp(R(x,t)/\beta)$ denotes the partition function ensuring normalization, and $R(x, t)$ is the reward function assumed to be finite.

Given a policy and its induced marginal distribution over answers, it is essential to equip model predictions with a meaningful confidence measure. Because our access is black-box, confidence must be estimated from sampled answers. We consider the majority voting procedure, or self-consistency~\citep{wang2023selfconsistency}, over independent reasoning traces: 

\begin{definition}[Majority Response, Confidence Estimator]\label{def:majority_estimator}\label{def:confidence_estimator}
Let $T_1, \dots, T_K$ be $K$ i.i.d. samples from a policy $\pi(\cdot \mid x)$, with their corresponding answers $f(T_k)$. The final answer is given by the majority response
$$
\hat{Y}_\pi^{(K)}(x) \coloneqq \arg\max_{y \in \mathcal{Y}} \sum_{k=1}^K \indicatrice{f(T_k) = y}
$$
following~\citet{kuhn2023semantic}, its associated confidence is the frequency of that response:
$$
\hat{c}_\pi(x) \coloneqq \frac{1}{K} \sum_{k=1}^K \indicatrice{f(T_k) = \hat{Y}_\pi^{(K)}(x)}.
$$
\end{definition}

After defining a confidence estimator arises the question of its calibration, i.e. how well it reflects accuracy. We first introduce a binary random variable indicating prediction correctness.
\begin{definition}[Accuracy Indicator]\label{def:accuracy_indicator}
Let $(X, Y)$ be a question–answer pair and $\hat{Y}^{(K)}_\pi(X)$ the majority-voting prediction from $K$ samples under policy $\pi$. The \emph{accuracy indicator} $I_\pi$ is 
$$
I_\pi \coloneqq \indicatrice{\hat{Y}^{(K)}_\pi(X) = Y},
$$

\end{definition}
Let $\mathbb{P} = \mathcal{D} \otimes \pi_T^{\otimes K}$ denote the joint probability measure over $(X, Y, T_1, \dots, T_K)$.
A model is \textit{perfectly calibrated} if, for any confidence $c \in [0, 1]$, the conditional probability of success matches the reported confidence \cite{guo2017calibration}:
\begin{equation}\label{eq:calibrated}
\mathbb{P}(I_\pi = 1 \mid \hat{c}_\pi = c) = c
\end{equation}
where  $\mathbb{P}(I_\pi = 1 \mid \hat{c}_\pi = c)$ represents the accuracy of the model for a given confidence score $c$.

In the theoretical analysis, we assess calibration using ECE~\citep{naeini2015obtaining}, which homogeneously penalizes misalignment between predicted probabilities and empirical accuracy:
\begin{equation}
\label{eq:calibration_error}
\text{ECE}(\pi) \coloneqq \mathbb{E} \left[ \big| \mathbb{P}(I_\pi = 1 \mid \hat{c}_\pi) - \hat{c}_\pi \big| \right]
\end{equation}

Previous works report that RL induces systemic \textit{overconfidence} in model outputs \cite{groot-valdenegro-toro-2024-overconfidence, epstein2025llmsoverconfidentevaluatingconfidence}, formalized as follows:

\begin{definition}[Overconfidence]\label{def:overconfidence}
Let $I_{\pi}$ denote the accuracy indicator and $\hat c_\pi$ the confidence estimator for a response to input $x$. The model is said to be overconfident if, almost surely:
$$
\mathbb{P}(I_{\pi} = 1 \mid \hat{c}_{\pi}) < \hat{c}_{\pi}
$$
\end{definition}

Such overconfidence increases ECE and degrades calibration. Since we operate under black-box access, we act on the input: we seek a transformation of the input question that acts as an increase of the KL-regularization parameter, reverting $\pi_\text{RL}$ toward $\pi_\text{ref}$, thereby reducing overconfidence.

\section{Theoretical Relaxation for UQ}\label{sec:J4U}

In this section, we formalize \textit{relaxation operators}, transformations of the input query that emulate an increase in the KL-regularization parameter during RL, and demonstrate that they improve calibration. In Section~\ref{sec:xp}, we propose jailbreak-style prompts as candidate instantiations of this mechanism.

\begin{definition}[Relaxation Operator]\label{def:relaxation}
Let $\pi_{\text{ref}}$ be a pre-trained reference policy and $\pi_{\text{RL}}$ be a policy derived from $\pi_\text{ref}$ via RL.
A relaxation operator $j$ is a transformation mapping an input $x$ to a modified version $j(x)$ such that the induced policy $\pi_j(\cdot \mid x) \coloneqq \pi_{\text{RL}}(\cdot \mid j(x)) $ is equivalent to a closed-form optimal policy in Eq.~\ref{eq:gibbs_policy} with a higher KL-regularization parameter. I.e. $\exists (\beta_1, \beta_2)$ s.t.:
$$
\pi_j \equiv \pi_{\beta_1}, \quad \text{with} \quad \pi_{\text{RL}} \equiv \pi_{\beta_2} \quad \text{and} \quad \beta_1 > \beta_2 
$$
We denote by $\mathcal{J}$ the family of such relaxation operators.
\end{definition}

It immediately follows that applying a relaxation operator shifts the policy back towards $\pi_\text{ref}$:

\begin{proposition}\label{prop:KL_divergence}
Let $\pi_{\text{ref}}$ be a pre-trained reference policy and $\pi_{\text{RL}}$ be a policy derived from $\pi_\text{ref}$ via RL. Let $j \in \mathcal{J}$ a relaxation operator inducing a policy $\pi_j$, 
\begin{equation}
\label{eq:relaxation_prop}
\mathbb{D}_{\text{KL}}\big(\pi_j || \pi_{\text{ref}}\big) < \mathbb{D}_{\text{KL}}\big(\pi_{\text{RL}} || \pi_{\text{ref}}\big)
\end{equation}
where $\mathbb{D}_{\text{KL}}$ denotes the Kullback-Leibler divergence. 
\end{proposition}

\begin{proof}[Proof Sketch]
Full proof in Appendix \ref{app:prop_proof}. $\pi_\beta$ is the unique solution to the KL-regularized objective $\min_{\pi} \mathcal{L}(\pi) = \mathbb{E}_\pi[-R(x,t)] + \beta \mathbb{D}_{\text{KL}}(\pi \parallel \pi_{\text{ref}})$. As $\beta$ increases, the penalty on the KL term grows relative to the reward. By constrained optimization, an increase in $\beta$ pulls the optimal solution closer to $\pi_{ref}$  in terms of KL divergence. Since $j$ induces $\beta_1 > \beta_2$, the inequality \eqref{eq:relaxation_prop} holds.
\end{proof}

Since overconfidence arises from excessive deviation from $\pi_{ref}$, this reversion property suggests improved calibration. We now formalize this intuition by showing ECE reduction.

\begin{theorem}\label{thm:relaxation_ece}
Consider a policy $\pi_{\text{RL}}  \equiv \pi_{\beta_2}$ an overconfident policy that underwent aggressive RL-based alignment, i.e.\ $\beta_2$ is small relative to reward differences.
There exists a $\beta'>\beta_2$ such that for any relaxation operator $j \in \mathcal{J}$
inducing a policy $\pi_j  \equiv \pi_{\beta_1} $, if $\beta_1\in(\beta_2,\beta']$ then the following holds for a finite number of $K$ traces samples:
\begin{equation}\label{thm:finite}
\text{ECE}(\pi_j) \le \text{ECE}(\pi_{\text{RL}}) + \mathcal{O}\left( \sqrt{\frac{\log |\mathcal{Y}|}{K}} \right)
\end{equation}
Furthermore, in the asymptotic limit where $K \to \infty$, the relaxation strictly reduces ECE:
\begin{equation}\label{thm:asymptotique}
\text{ECE}(\pi_j) < \text{ECE}(\pi_{\text{RL}})
\end{equation}
\end{theorem}

\begin{proof}[Proof Sketch]
Complete proof in Appendix~\ref{app:thm_proof}.
Under RL overalignment, there are relaxations that will smooth rewards without changing the theoretical mode, which locally maximize the reward. 
In such cases,
we show that calibration improves under relaxation by first decomposing the ECE in an overconfident regime as $\text{ECE}(\pi) = \mathbb{E}[\hat{c}_\pi] - \mathbb{P}(I_\pi = 1)$. 
Using a standard maximal inequality for sub-Gaussian variables \citep{boucheron2013inequalities}, we bound the finite-sample bias of the confidence estimator $\hat{c}_\pi$ by $\mathcal{O}\left(\sqrt{\log |\mathcal{Y}|/K}\right)$, which vanishes as the number of samples $K \to \infty$. 
For an optimal policy $\pi_\beta$, the derivative of the theoretical mode probability $y^*_\pi$ with respect to  $\beta$ satisfies $\partial_\beta \pi_\beta(y^*_\pi|x) \propto (\mathbb{E}[R] - \mathbb{E}[R|y^*_\pi])$. 
Since under RL overalignment, this $y^*_\pi$ corresponds to the mode maximizing the expected reward, $\mathbb{E}[R\vert{}y^*_\pi] \ge \mathbb{E}[R]$, making this derivative non-positive.
Consequently, increasing $\beta$ reduces the expected confidence and thereby lowering the ECE.
\end{proof}

\section{Jailbreak-style transformations for UQ: Calibration and Behavioral evidence of relaxation}\label{sec:xp}

In this section, we introduce \textsc{J4U} operators, evaluate their accuracy and calibration against state-of-the-art baselines, then test whether their behavior matches the relaxation-operator predictions

\subsection{\textsc{J4U} as candidate instantiations of relaxation}

Jailbreaks are designed to bypass alignment-induced constraints while preserving prompts' semantic~\citep{zhou2023lima,10.5555/3666122.3669630,10.5555/3692070.3694246}. They tend to elicit responses that more closely reflect the model's pre-trained knowledge distribution rather than its safety-tuned persona, making them inherently good candidates for relaxation.

We operationalize this intuition through a stochastic operator $j$ that transforms an input $x$ into a jailbreak-style variant $j(x)$ designed to (i) preserve the payload of 
$x$ and (ii) mitigate alignment-induced constraints,  thereby increasing response diversity and reducing overconfident mass concentration. We estimate uncertainty through $K$ sampling of such transformations $j(x)$. We instantiate $j$ using three representative one-shot black-box jailbreak techniques, spanning major attack classes identified by~\citet{shen2025safetyinstinctsllmslearn} and~\citet{li-etal-2024-drattack}. Each instantiation defines a distribution over perturbations, from which we sample independently at each query.
\begin{itemize}

        \item \textbf{\textsc{J4U-Suffix}}: a random character suffix is appended to the prompt. It is a simplification of gradient-based discrete optimization methods~\citep{shen2025safetyinstinctsllmslearn}, specifically  universal suffix attacks~\citep{zou2023universaltransferableadversarialattacks},  representative of ``Suffix-based'' attacks~\citep{li-etal-2024-drattack}.
        \item \textbf{\textsc{J4U-Art}~\citep{jiang-etal-2024-artprompt}}: a random word is replaced by its ASCII-art equivalent. It is representative of ``Semantic obfuscation'' attacks where prompts are perturbed through non-standard formatting or localized obfuscations.
        \item \textbf{\textsc{J4U-Prog}~\citep{10579515}}: the question is split into subsections reconstructed by the LRM before answering. It is representative of ``Template-based restructuring'' attacks where prompts are embedded into structured templates or multi-step reasoning frameworks.
\end{itemize}

\subsection{Experimental setup}

\noindent\textbf{Evaluated black-box UQ methods.} We conduct empirical evaluation using the \textsc{J4U}  instantiations introduced above (\textsc{J4U-Suffix}, \textsc{J4U-Art}, \textsc{J4U-Prog}), two state-of-the-art black-box UQ approaches and a simple baseline:

\begin{itemize}

\item \textbf{\textsc{RS}  (repeat-sampling baseline)}: We compare all approaches with regard to a simple repeated-sampling baseline, representing $\pi_\text{RL}$, where the input is repeated.

\item \textbf{\textsc{VC} (verbalized confidence)}: Following the Avg-Conf aggregation strategy and self-random and self-probing settings  of~\citet{xiong2023can}, the model is queried to produce an answer and a separate instance is then prompted to produce a confidence score for said question and answer. The confidence score is a normalized average over several iterations.

\item \textbf{\textsc{Rephrase} (self-consistency with rephrasing)}: Several prior approaches introduce prompt perturbations to increase variance in self-consistency. We adopt the method of \citet{yang2024just}, which rephrases prompts to induce diverse generations. 

\end{itemize}

\noindent\textbf{LRMs.}
Experiments are conducted on \gpt~\citep{openai2025gptoss120bgptoss20bmodel}, \qwen~\citep{yang2025qwen3technicalreport}, \deepseek~\citep{deepseek}, and \closedgpt~\citep{openai2026gpt56}, a proprietary LRM from OpenAI's GPT-5.6 frontier-model family. We select three state-of-the-art open-weight LRMs for reproducibility, plus GPT-5.6 Luna to test generalization to production settings.

\noindent\textbf{Datasets.} We evaluate our method and baselines on three datasets. Two are standard QA benchmarks designed to assess factual knowledge and reasoning, \textbf{SuperGPQA}~\citep{du2025supergpqa}  and \textbf{Humanity’s Last Exam (HLE)}~\citep{hle}. The third, \textbf{SynthPAI}~\citep{yukhymenko2024synthetic}, consists of synthetic Reddit posts typically used to evaluate attribute inference attacks~\citep{staab24beyond}, a task with significant aleatoric uncertainty where the goal is to infer implicit attributes of a text’s author.  While we focus on QA, prior works have shown that QA formulation can cover a broad range of NLP phenomena (e.g., \citet{mccann2018naturallanguagedecathlonmultitask}, \citet{khashabi-etal-2020-unifiedqa}).

\noindent\textbf{Confidence calibration metrics.} While our theoretical analysis focuses on ECE~\citep{naeini2015obtaining}, 
we extend the empirical evaluation to NLL~\citep{fisher1922mathematical} and the Brier score~\citep{glenn1950verification}.  
NLL emphasizes the model’s confidence near certainty, penalizing overconfident errors, 
while the Brier score captures overall discrepancy between predicted probabilities and outcomes.  
Together with ECE, these metrics provide a complementary view of calibration.

\noindent\textbf{Sampling parameters and statistical setting.} We report one-shot results aggregated over $K=10$ samples per question for each UQ method (see Appendix~\ref{app:confidencePerIteration} for a study varying $K$). We note that all evaluated methods rely on multiple sampling rounds. The cost analysis in Appendix~\ref{app:cost} shows that \textsc{J4U} is substantially more efficient than \textsc{VC} and \textsc{Rephrase}, both asymptotically and empirically, in terms of API calls, token consumption, and wall-clock execution time. Statistical significance is assessed with a two-sided bootstrap test on the mean difference (method minus \textsc{RS}), using $\alpha=0.05$. We perform the test separately for each (LRM, dataset) pair and metric. For more details on the experimental setup and prompts, see Appendix~\ref{app:detailed_XP} and~\ref{app:prompt}.

\subsection{J4U Improves Confidence Calibration}

\begin{table}[t]
\centering
\caption{
Comparison versus \textsc{RS} (across 12 LRM-dataset pairs). Left: relative (\%) and absolute ($\Delta$) average differences; bold values denote the best-performing approach per metric.
Variation per LRMs and dataset is encoded as a 4×3 heatmap (rows = LRMs, columns = datasets): green = improvement, red = degradation; darker shades indicate statistical significance.  
}

\label{tab:summary_percentage_metrics}

\setlength{\tabcolsep}{1pt}

\begin{tabular}{l| r r r r r}
\toprule
 \multicolumn{1}{c}{}& \multicolumn{1}{c}{\textsc{VC}} & \multicolumn{1}{c}{\textsc{Rephrase}}& \multicolumn{1}{c}{\textsc{J4U-Suffix}} & \multicolumn{1}{c}{\textsc{J4U-Prog}} & \multicolumn{1}{c}{\textsc{J4U-Art}} \\
\multicolumn{1}{c}{} & \multicolumn{1}{c}{\small ~\citep{xiong2023can}} & \multicolumn{1}{c}{\small~\citep{yang2024just}}  & \multicolumn{1}{c}{(ours)} & \multicolumn{1}{c}{(ours)} & \multicolumn{1}{c}{(ours)} \\
\midrule
Acc.$\uparrow$     
& \cellheatmap{{1,0,-1,-1,0,0,0,1,1,-1,-1,-1}}{-1.2\%}{-0.005 $\Delta$} 
& \cellheatmap{{0,-1,1,-1,-1,-1,-1,-2,1,-1,0,-1}}{+1.4\%}{-0.025 $\Delta$} 
& \cellheatmap{{0,-1,1,-1,-1,1,0,0,1,-1,0,-1}}{+2.8\%}{\textbf{-0.001} $\Delta$} 
& \cellheatmap{{1,-1,-1,-1,-1,-1,1,-1,1,-1,-1,-1}}{-1.3\%}{-0.017 $\Delta$} 
& \cellheatmap{{1,-1,0,-1,-1,1,1,-1,2,0,-1,-1}}{\textbf{+7.0\%}}{-0.008 $\Delta$}  \\
ECE$\downarrow$    
& \cellheatmap{{1,-1,-1,-1,-1,-1,-1,-1,0,-1,-1,-2}}{+8.6\%}{+0.035 $\Delta$} 
& \cellheatmap{{1,0,1,0,2,0,-1,1,2,-1,0,-1}}{-3.1\%}{-0.018 $\Delta$}
& \cellheatmap{{1,-1,-1,1,0,1,1,1,1,1,0,-1}}{-2.9\%}{-0.008 $\Delta$}
& \cellheatmap{{2,2,2,1,2,1,2,1,2,1,1,1}}{\textbf{-15.8\%}}{\textbf{-0.068} $\Delta$}
& \cellheatmap{{2,2,2,1,1,2,2,-1,2,2,1,1}}{-13.8\%}{-0.066 $\Delta$} \\
Brier$\downarrow$  
& \cellheatmap{{1,0,-1,-1,-1,-1,0,0,-1, -1, 0, -1}}{+3.5\%}{+0.008 $\Delta$}
& \cellheatmap{{1,0,1,-1,0,0,0,-2,1,-1,0,0}}{+2.0\%}{+0.002 $\Delta$}
& \cellheatmap{{1,0,-1,1,0,0,1,0,1,1,0,0}}{-1.7\%}{-0.006 $\Delta$}
& \cellheatmap{{2,1,2,0,0,1,1,-1,2,1,0,0}}{-4.4\%}{-0.015 $\Delta$}
& \cellheatmap{{2,0,2,1,0,2,2,-1,2,1,0,0}}{\textbf{-4.8\%}}{\textbf{-0.018} $\Delta$} \\
NLL$\downarrow$    
& \cellheatmap{{-1,-1,1,-1,1,-1,-1,-1,1,-1,-1,-1}}{+5.0\%}{+0.55 $\Delta$}
& \cellheatmap{{1,-1,1,1,1,-1,2,-2,1,-1,1,1}}{-2.0\%}{-0.24 $\Delta$}
& \cellheatmap{{2,-1,1,2,-1,1,1,-1,1,2,1,-2}}{-6.4\%}{-0.30 $\Delta$} 
& \cellheatmap{{2,1,2,1,1,1,2,-1,1,1,1,1}}{-14.4\%}{-1.73 $\Delta$}
& \cellheatmap{{2,-1,2,1,1,1,2,-1,2,2,1,1}}{\textbf{-14.5\%}}{\textbf{-1.80} $\Delta$} \\
\bottomrule
\end{tabular}
\vspace{-0.2cm}
\end{table}
In this subsection, we compare the relative accuracy and confidence calibration of \textsc{J4U} and state-of-the-art approaches with reference to $\pi_\text{RL}$ induced by \textsc{RS} (repeated sampling baseline). 

Table~\ref{tab:summary_percentage_metrics} reports the average gain in accuracy and confidence calibration over \textsc{RS} at $K=10$ for each method. Bold values indicate the best-performing method for a given metric. Each value is accompanied by a $4 \times 3$ heatmap summarizing results across LRM–dataset combinations: rows correspond to LRMs (from top to bottom: \qwen, \gpt, \deepseek, \closedgpt) and columns to datasets (from left to right: SynthPAI, SuperGPQA, HLE). Each cell represents the relative deviation from \textsc{RS}, with green indicating improvement and red indicating degradation. Darker shades denote statistically significant differences as determined by a bootstrap mean difference test ($\alpha = 0.05$). Across 12 LRM-dataset pairs and 3 calibration metrics (ECE, Brier, NLL), this yields 36 evaluation settings. Results per dataset and LRM are reported in Appendix~\ref{app:detailacc}.

\noindent\textbf{Result 1: Observed gains of \textsc{VC} and \textsc{Rephrase} relative to \textsc{RS} are limited in our settings.} 
\textsc{VC} shows no statistically significant improvement over \textsc{RS} in any of the calibration evaluations and yields average degradations on each calibration metric.  
\textsc{Rephrase} yields statistically significant improvements in 3 of 36 calibration evaluations and statistically significant degradations in 3 others. On average, \textsc{Rephrase} reduces ECE by 3.1\% and NLL by 2\%, while it increases Brier by 2\%.

\noindent\textbf{Result 2: \textsc{J4U-Art} and \textsc{J4U-Prog} improve calibration in a larger fraction of settings than baselines, and by a wider margin, without significant accuracy loss.} 
\textsc{J4U-Prog} and \textsc{J4U-Art} achieve statistically significant ECE improvements in $6$ and $7$ of the 12 LRM–dataset pairs, respectively. Across all calibration metrics, \textsc{J4U-Art} shows statistically significant improvement in $18/36$ cases and \textsc{J4U-Prog} in $12/36$ cases.
No jailbreak transformation yields a statistically significant accuracy decrease, and \textsc{J4U-Art} significantly improves accuracy on \deepseek-HLE. Relative accuracy gains are driven by HLE, where the \textsc{RS} baseline accuracy is low; in absolute terms, all methods remain within 0.025 of \textsc{RS}.

Beyond occurring more often, \textsc{J4U-Art}'s and \textsc{J4U-Prog}'s gains are also substantially larger in magnitude: on average, \textsc{J4U-Prog} reduces ECE by 15.8\% (13.8\% for \textsc{J4U-Art}), while \textsc{J4U-Art} achieves the largest reduction in Brier (4.8\%) and NLL (14.5\%) (4.4\% and 14.4\% for \textsc{J4U-Prog}).

\textit{We next investigate whether these empirical findings are consistent with the relaxation-based interpretation, by testing observable implications motivated by Theorem~\ref{thm:relaxation_ece}.}

\subsection{ Relaxation-operator predictions and observed behaviors in \textsc{J4U}}

Theorem~\ref{thm:relaxation_ece} relies on 2 assumptions: (i) $\pi_\text{RL}$ exhibits systemic overconfidence and is the result of over-aggressive RL, and (ii) $j$ is a relaxation operator. By hypothesis, we assume that LRMs undergo aggressive RL and, in line with prior studies~\citep{epstein2025llmsoverconfidentevaluatingconfidence,groot-valdenegro-toro-2024-overconfidence,he2023investigating}, we find that tested LRMs are overconfident --assumption (i)-- (Appendix~\ref{app:overconfidence}).

Hence, we seek to assess whether jailbreaks transformations exhibit observable behaviors consistent with relaxation --assumption (ii)--. Since relaxation cannot be verified directly, we test two observable implications of Definition~\ref{def:relaxation}: 
(1) a flattening of the predictive distribution (higher Shannon entropy) and 
(2) an increase in the semantic diversity in reasoning traces (lower cosine similarity).

\noindent\textbf{Implication 1: Flattening the predictive distribution.} 
As shown in the monotonicity proof (Appendix \ref{proof:monotonie}), the probability of the theoretical mode $\pi(y^*_\pi \mid x)$ is a strictly decreasing function of $\beta$. Since the distribution remains normalized, this probability mass is necessarily redistributed toward the tails of the distribution. This leads to a flatter probability vector, which directly results in a higher Shannon entropy. Hence, if \textsc{J4U} instantiations act as relaxation operators, we expect derived policies to  shift back toward the higher-entropy reference policy $\pi_{\text{ref}}$.

\begin{table}[t]
\centering
\caption{Increase in entropy vs $\pi_\text{RL}$, in relative change (in \%) and  absolute difference ($\Delta$).}
\small
\setlength{\tabcolsep}{4pt}
\begin{tabular}{l ccc}
\toprule
LRM & \textsc{J4U-Suffix} & \textsc{J4U-Prog} & \textsc{J4U-Art} \\
\midrule
\qwen
& \makecell{+7.6\% \scriptsize +0.015 $\Delta$}
& \makecell{+102.9\% \scriptsize +0.308 $\Delta$}
& \makecell{+72.5\% \scriptsize +0.210 $\Delta$} \\
\gpt
& \makecell{+9.2\% \scriptsize +0.030 $\Delta$}
& \makecell{+39.9\% \scriptsize +0.155 $\Delta$}
& \makecell{+17.4\% \scriptsize +0.069 $\Delta$} \\
\deepseek
& \makecell{+13.2\% \scriptsize +0.035 $\Delta$}
& \makecell{+57.7\% \scriptsize +0.158 $\Delta$}
& \makecell{+38.6\% \scriptsize +0.106 $\Delta$} \\
\closedgpt
& \makecell{+17.0\% \scriptsize +0.018 $\Delta$}
& \makecell{+57.9\% \scriptsize +0.109 $\Delta$}
& \makecell{+62.4\% \scriptsize +0.114 $\Delta$} \\
\midrule
Global Mean
& \makecell{+11.8\% \scriptsize +0.024 $\Delta$}
& \makecell{+64.6\% \scriptsize +0.182 $\Delta$}
& \makecell{+47.7\% \scriptsize +0.125 $\Delta$} \\
\bottomrule
\end{tabular}
\label{tab:summary_entropy_combined}
\end{table}

Table~\ref{tab:summary_entropy_combined} provides a summary of the increase in Shannon entropy induced by \textsc{J4U} relative to $\pi_\text{RL}$, the policy defined by \textsc{RS}. We report both absolute and relative average increase per LRM. Detailed results are given in Appendix~\ref{app:entropy}.

\noindent\textbf{Result 3.1: \textsc{J4U} methods shift $\pi_\text{RL}$ toward flatter predictive probability distributions.}
Across all LRMs and datasets, \textsc{J4U} consistently increases entropy, producing higher predictive variability than \textsc{RS}. The magnitude of the shift is highly dependent on the transformation. While \textsc{J4U-Suffix} yields a modest average entropy increase of +11.8\%,  \textsc{J4U-Prog} and \textsc{J4U-Art} induce much more significant shifts, with average increases of +64.6\% and +47.7\%, respectively. This is consistent with previous results showing that \textsc{J4U-Suffix} is less effective.

The reaction to each method varies by LRM. \gpt~appears least affected, which is consistent with the smaller confidence calibration gains achieved by our methods on this model (see Appendix~\ref{app:detailacc}). In contrast, \closedgpt~appears comparatively more responsive to \textsc{J4U}, reaching roughly a +60\% increase for \textsc{J4U-Prog} and \textsc{J4U-Art}. It trails only \qwen, which appears particularly sensitive to \textsc{J4U-Prog}, showing a +102.9\% increase in entropy. 

\textit{These results provide a first layer of empirical evidence that the
observable consequences of \textsc{J4U} match those predicted for a relaxation
operator. This
is further supported in what follows.}

\noindent\textbf{Implication 2: Increasing semantic variability in reasoning traces.}
A more direct consequence of increasing the KL-regularization parameter $\beta$ in the closed-form optimal policy (Eq.~\ref{eq:gibbs_policy}) is a rise in semantic diversity of the generated reasoning traces. This effect reflects the fundamental \textit{exploration–exploitation} trade-off: by penalizing excessive concentration on narrow modes of the reward function $R(x,t)$, stronger regularization preserves the entropy of the reference distribution $\pi_{\text{ref}}$. As a result, higher regularization promotes a broader exploration of the semantic solution space. Reasoning traces are not accessible under our theoretical black-box model and are not required by \textsc{J4U}; we analyze them here solely to better understand its empirical behavior. \closedgpt~was excluded from this experiment as its reasoning traces are not accessible.

To quantify this effect, we compute the average pairwise cosine similarity between reasoning trace embeddings generated using \texttt{Qwen3-Embedding-8B}
~\citep{qwen3embedding}, without any system prompt. We report both absolute and relative variations with respect to $\pi_\text{RL}$, the policy induced by \textsc{RS}. A summary is provided in Table~\ref{tab:summary_cosine_combined}, with full results deferred to Appendix~\ref{app:cot}.

\begin{table}[]
\centering
\caption{Reduction in reasoning traces cosine similarity vs $\pi_\text{RL}$, in relative change (in \%) and absolute difference ($\Delta$), with the exclusion of \closedgpt~(reasoning traces are not accessible).}
\small
\setlength{\tabcolsep}{4pt}
\begin{tabular}{l ccc}
\toprule
LRM & \textsc{J4U-Suffix} & \textsc{J4U-Prog} & \textsc{J4U-Art} \\
\midrule
\qwen
& \makecell{-5.7\% \scriptsize -0.037 $\Delta$} 
& \makecell{-16.8\% \scriptsize -0.113 $\Delta$} 
& \makecell{-28.8\% \scriptsize -0.187 $\Delta$} \\
\gpt
& \makecell{-1.1\% \scriptsize -0.013 $\Delta$} 
& \makecell{-22.3\% \scriptsize -0.183 $\Delta$} 
& \makecell{-31.5\% \scriptsize -0.250 $\Delta$} \\
\deepseek
& \makecell{-3.7\% \scriptsize -0.017 $\Delta$} 
& \makecell{-4.8\% \scriptsize -0.037 $\Delta$} 
& \makecell{-23.7\% \scriptsize  -0.147 $\Delta$} \\
\midrule
Global Mean
& \makecell{-3.5\% \scriptsize -0.022 $\Delta$} 
& \makecell{-14.6\% \scriptsize -0.111 $\Delta$} 
& \makecell{-28.0\% \scriptsize -0.195 $\Delta$} \\
\bottomrule
\end{tabular}
\label{tab:summary_cosine_combined}
\end{table}

\noindent\textbf{Result 3.2: Jailbreaks increase reasoning traces semantic diversity.} 
Consistent with stronger regularization, \textsc{J4U} reduces cosine similarity across all LRMs, indicating increased semantic dispersion in the generated reasoning traces. In line with previous observations, \textsc{J4U-Suffix} induces only limited variability (-3.5\% on average), whereas \textsc{J4U-Prog} and \textsc{J4U-Art} yield larger reductions averaging -14.6\% and -28.0\%, respectively. \textsc{J4U-Prog} affects LRMs differently, with a reduction ranging from only -4.8\% on \deepseek~to -22.3\% on \gpt. Overall, \textsc{J4U-Art} consistently produces the greatest semantic diversification with more stability across LRMs, ranging from -23.7\% on \deepseek~to -31.5\% on \gpt.

\textit{
Thus, at the level of both outputs and internal reasoning traces, \textsc{J4U}
produces the two observable signatures predicted for a relaxation operator:
higher-entropy predictive distributions and greater semantic dispersion across
reasoning traces. While this does not establish that \textsc{J4U} instantiates
a relaxation operator, it provides
converging indirect evidence in support of this interpretation.}

\section{Conclusion and Discussion}
\label{sec:discu}

We introduced \textsc{J4U}, a black-box uncertainty quantification method for LRMs that acts through \emph{jailbreak-based prompt-level perturbations}. Motivated by a RL view of alignment, we argued that it can suppress useful predictive variability, making many existing black-box UQ techniques hard to distinguish from simple repeated sampling. To address this, we introduced \emph{relaxation operators}, input transformations that theoretically emulate a higher KL-regularization parameter and, under the aggressive alignment assumption, provably reduce ECE (Sec.~\ref{sec:J4U}). We proposed jailbreak-style transformations as candidate instantiations of this mechanism and found that, empirically, they reproduce the observable signatures predicted by relaxation while improving confidence calibration (Sec.~\ref{sec:xp}).

\noindent\textbf{Discussion.} Perturbations may not restore \textit{legitimate} predictive variability in regions of uncertainty, but instead introduce generic variability that would not reflect true model uncertainty. To address this concern, we evaluate our method on \texttt{gsm8k}, a benchmark where all 4 tested LRMs achieve high accuracy (95\%, see Appendix~\ref{app:easybench}). In this setting, existing UQ methods are already well-calibrated and high-confidence is justified, providing a stringent test for spurious variability injection. Our results show that, unlike a random noise baseline, \textsc{J4U} does not significantly degrade accuracy or ECE. This suggests that \textsc{J4U} does not merely introduce random perturbations, but largely preserves regions of legitimate certainty while improving uncertainty estimates elsewhere.

A limitation of our work is its reliance on jailbreak-style adversarial prompts, behaviors that model providers are expected to mitigate. This creates a risk that specific jailbreaks may become ineffective over time. However, we show that \textsc{J4U} remains effective across jailbreak families, suggesting that the method is not tied to a single brittle attack pattern and can adapt to novel jailbreaking techniques. Recent surveys indicate that jailbreak methods are diverse, evolving, and steadily improving~\cite{yi2024jailbreakattacksdefenseslarge}, while existing defenses remain ineffective even for recent models~\cite{DBLP:conf/acl/XuLDLP24,pathade2025redteamingmindmachine}. Furthermore, theoretical results  suggest that as long as a model assigns non-zero probability to policy-violating behaviors, there exists a prompt capable of eliciting it~\cite{10.5555/3692070.3694246}.

\noindent\textbf{Future Work.} 
A natural extension of this study would be to explore more permissive access levels.
Access to internal weights and gradients could enable a mechanistic analysis of alignment and, in turn, identifying and specifically relaxing the model components or layers most responsible for predictive entropy collapse, potentially leading to even more robust calibration techniques.

\subsection*{AI use statement}

In this work, we used generative AI tools to edit portions of the manuscript text for clarity and readability, to identify relevant literature, and to assist with coding for experiments and plotting. We have not used generative AI tools to formulate mathematical claims, develop proofs, or design the core theoretical framework of relaxation operators, which were conceived and derived entirely by the authors. We have reviewed all AI-assisted work: all authors reviewed the final manuscript, suggested literature was read by at least one author, produced code was proofread and its outputs checked for consistency by at least one author. We take responsibility for the final content of this work, including text, claims, or artifacts produced with the aid of generative AI.

\subsection*{Ethics statement}

This work leverages existing jailbreak-style techniques for a legitimate purpose: improving black-box uncertainty quantification for LRMs. In our experiments, jailbreak-style transformations are applied exclusively to strictly benign questions, and are never used to elicit harmful or policy-violating content. We do not propose new jailbreak attacks, nor do we improve the effectiveness of existing ones at eliciting harmful content. Our contribution is a novel repurposing of these techniques for a utility goal, not an advance in jailbreak attack capability.

\subsection*{Reproducibility statement}

The assumptions underlying our theoretical results are stated in
Section~\ref{sec:J4U} and complete proofs are reported in Appendix~\ref{app:all_proofs}. Experimental settings are detailed in Appendix~\ref{app:detailed_XP}, and
exact prompts used for every method on one of the datasets, including the specification of the three
J4U operators, are given verbatim in Appendix~\ref{app:prompt}. All datasets
are publicly available and referenced with their source; no additional preprocessing beyond the step
described in Appendix~\ref{app:detailed_XP} is applied. All code used to produce the reported results,
including the J4U transformations, evaluation metrics and statistical tests,
is provided as anonymized supplementary material and will be released in a
public repository upon acceptance. Results on the three open-weight LRMs are
fully reproducible from this material; results on \closedgpt\ depend
on a proprietary API and may vary with future model updates.


\subsubsection*{Acknowledgments}

We thanks colleagues from Inria Saclay for their feedback on the paper, and particularly Pietro Marco Congedo for his opinion on early versions of the proofs presented in appendix. This work was partially supported by grant ANR-22-PECY-0002 (IPoP project) of the Agence Nationale de la Recherche. This project was provided with computing AI and storage resources by GENCI at IDRIS thanks to the grant 2026-AD011016360 on the supercomputer Jean Zay's H100 partition.

\bibliography{iclr2025_conference}

@article{shorinwa2025survey,
author = {Shorinwa, Ola and Mei, Zhiting and Lidard, Justin and Ren, Allen Z. and Majumdar, Anirudha},
title = {A Survey on Uncertainty Quantification of Large Language Models: Taxonomy, Open Research Challenges, and Future Directions},
year = {2025},
issue_date = {February 2026},
publisher = {Association for Computing Machinery},
address = {New York, NY, USA},
volume = {58},
number = {3},
issn = {0360-0300},
url = {https://doi.org/10.1145/3744238},
doi = {10.1145/3744238},
journal = {ACM Comput. Surv.},
month = sep,
articleno = {63},
numpages = {38}
}

@article{10.1093/jamiaopen/ooae154,
    author = {Gao, Yanjun and Myers, Skatje and Chen, Shan and Dligach, Dmitriy and Miller, Timothy and Bitterman, Danielle S and Chen, Guanhua and Mayampurath, Anoop and Churpek, Matthew M and Afshar, Majid},
    title = {Uncertainty estimation in diagnosis generation from large language models: next-word probability is not pre-test probability},
    journal = {JAMIA Open},
    volume = {8},
    number = {1},
    pages = {ooae154},
    year = {2025},
    month = {01},
    issn = {2574-2531},
    doi = {10.1093/jamiaopen/ooae154},
    url = {https://doi.org/10.1093/jamiaopen/ooae154},
    eprint = {https://academic.oup.com/jamiaopen/article-pdf/8/1/ooae154/61411131/ooae154.pdf},
}

@article{he2023investigating,
  title={Investigating uncertainty calibration of aligned language models under the multiple-choice setting},
  author={He, Guande and Cui, Peng and Chen, Jianfei and Hu, Wenbo and Zhu, Jun},
  journal={arXiv preprint arXiv:2310.11732},
  year={2023}
}

@article{doi:10.1056/AIp2300031,
author = {Alexander V. Eriksen  and Sören Möller  and Jesper Ryg },
title = {Use of GPT-4 to Diagnose Complex Clinical Cases},
journal = {NEJM AI},
volume = {1},
number = {1},
pages = {AIp2300031},
year = {2024},
doi = {10.1056/AIp2300031},

URL = {https://ai.nejm.org/doi/full/10.1056/AIp2300031},
eprint = {https://ai.nejm.org/doi/pdf/10.1056/AIp2300031}
}

@inproceedings{groot-valdenegro-toro-2024-overconfidence,
    title = "Overconfidence is Key: Verbalized Uncertainty Evaluation in Large Language and Vision-Language Models",
    author = "Groot, Tobias  and
      Valdenegro - Toro, Matias",
    editor = "Ovalle, Anaelia  and
      Chang, Kai-Wei  and
      Cao, Yang Trista  and
      Mehrabi, Ninareh  and
      Zhao, Jieyu  and
      Galstyan, Aram  and
      Dhamala, Jwala  and
      Kumar, Anoop  and
      Gupta, Rahul",
    booktitle = "Proceedings of the 4th Workshop on Trustworthy Natural Language Processing (TrustNLP 2024)",
    month = jun,
    year = "2024",
    address = "Mexico City, Mexico",
    publisher = "Association for Computational Linguistics",
    url = "https://aclanthology.org/2024.trustnlp-1.13/",
    doi = "10.18653/v1/2024.trustnlp-1.13",
    pages = "145--171"
}

@InProceedings{10.1007/978-3-032-01399-6_5,
author={Nariman Naderi and Zahra Atf and Peter R Lewis and Aref Mahjoub far and Seyed Amir Ahmad Safavi-Naini and Ali Soroush},
editor="Calvaresi, Davide
and Najjar, Amro
and Omicini, Andrea
and Aydogan, Reyhan
and Carli, Rachele
and Ciatto, Giovanni
and Tiribelli, Simona
and Fr{\"a}mling, Kary",
title="Evaluating Prompt Engineering Techniques for Accuracy and Confidence Elicitation in Medical {LLM}s",
booktitle="Explainable, Trustworthy, and Responsible AI and Multi-Agent Systems",
year="2026",
publisher="Springer Nature Switzerland",
address="Cham",
pages="67--84",

}

@misc{yang2024verbalizedconfidencescoresllms,
      title={On Verbalized Confidence Scores for {LLM}s}, 
      author={Daniel Yang and Yao-Hung Hubert Tsai and Makoto Yamada},
      year={2024},
      eprint={2412.14737},
      archivePrefix={arXiv},
      primaryClass={cs.CL},
      url={https://arxiv.org/abs/2412.14737}, 
}

@INPROCEEDINGS{10579515,
  author={Kang, Daniel and Li, Xuechen and Stoica, Ion and Guestrin, Carlos and Zaharia, Matei and Hashimoto, Tatsunori},
  booktitle={2024 IEEE Security and Privacy Workshops (SPW)}, 
  title={Exploiting Programmatic Behavior of {LLM}s: Dual-Use Through Standard Security Attacks}, 
  year={2024},
  volume={},
  number={},
  pages={132-143},
  doi={10.1109/SPW63631.2024.00018}}

@inproceedings{li-etal-2024-drattack,
    title = "{D}r{A}ttack: Prompt Decomposition and Reconstruction Makes Powerful {{LLM}}s Jailbreakers",
    author = "Li, Xirui  and
      Wang, Ruochen  and
      Cheng, Minhao  and
      Zhou, Tianyi  and
      Hsieh, Cho-Jui",
    editor = "Al-Onaizan, Yaser  and
      Bansal, Mohit  and
      Chen, Yun-Nung",
    booktitle = "Findings of the Association for Computational Linguistics: EMNLP 2024",
    month = nov,
    year = "2024",
    address = "Miami, Florida, USA",
    publisher = "Association for Computational Linguistics",
    url = "https://aclanthology.org/2024.findings-emnlp.813/",
    doi = "10.18653/v1/2024.findings-emnlp.813",
    pages = "13891--13913"
}

@misc{zhao2026seseblackboxuncertaintyquantification,
      title={SeSE: Black-Box Uncertainty Quantification for Large Language Models Based on Structural Information Theory}, 
      author={Xingtao Zhao and Hao Peng and Dingli Su and Xianghua Zeng and Chunyang Liu and Jinzhi Liao and Philip S. Yu},
      year={2026},
      eprint={2511.16275},
      archivePrefix={arXiv},
      primaryClass={cs.CL},
      url={https://arxiv.org/abs/2511.16275}, 
}

@inproceedings{manakul-etal-2023-selfcheckgpt,
    title = "{S}elf{C}heck{GPT}: Zero-Resource Black-Box Hallucination Detection for Generative Large Language Models",
    author = "Manakul, Potsawee  and
      Liusie, Adian  and
      Gales, Mark",
    editor = "Bouamor, Houda  and
      Pino, Juan  and
      Bali, Kalika",
    booktitle = "Proceedings of the 2023 Conference on Empirical Methods in Natural Language Processing",
    month = dec,
    year = "2023",
    address = "Singapore",
    publisher = "Association for Computational Linguistics",
    url = "https://aclanthology.org/2023.emnlp-main.557/",
    doi = "10.18653/v1/2023.emnlp-main.557",
    pages = "9004--9017"
}

@inproceedings{portillo-wightman-etal-2023-strength,
    title = "Strength in Numbers: Estimating Confidence of Large Language Models by Prompt Agreement",
    author = "Portillo Wightman, Gwenyth  and
      Delucia, Alexandra  and
      Dredze, Mark",
    editor = "Ovalle, Anaelia  and
      Chang, Kai-Wei  and
      Mehrabi, Ninareh  and
      Pruksachatkun, Yada  and
      Galystan, Aram  and
      Dhamala, Jwala  and
      Verma, Apurv  and
      Cao, Trista  and
      Kumar, Anoop  and
      Gupta, Rahul",
    booktitle = "Proceedings of the 3rd Workshop on Trustworthy Natural Language Processing (TrustNLP 2023)",
    month = jul,
    year = "2023",
    address = "Toronto, Canada",
    publisher = "Association for Computational Linguistics",
    url = "https://aclanthology.org/2023.trustnlp-1.28/",
    doi = "10.18653/v1/2023.trustnlp-1.28",
    pages = "326--362",
}

@inproceedings{JiangCalibratingLM,
  title={Calibrating Language Models via Augmented Prompt Ensembles},
  author={Mingjian Jiang and Yangjun Ruan and Sicong Huang and Saifei Liao and Silviu Pitis and Roger Baker Grosse and Jimmy Ba},
  booktitle = "ICML 2023 Workshop on Deployable Generative AI",
  year = "2023",
  url={https://openreview.net/pdf?id=L0dc4wqbNs}
}

@inproceedings{10.5555/3692070.3694246,
author = {Wolf, Yotam and Wies, Noam and Avnery, Oshri and Levine, Yoav and Shashua, Amnon},
title = {Fundamental limitations of alignment in large language models},
year = {2024},
publisher = {JMLR.org},
booktitle = {Proceedings of the 41st International Conference on Machine Learning},
articleno = {2176},
numpages = {34},
location = {Vienna, Austria},
series = {ICML'24}
}

@article{xiong2023can,
title={Can {{LLM}}s Express Their Uncertainty? An Empirical Evaluation of Confidence Elicitation in {{LLM}}s},
author={Miao Xiong and Zhiyuan Hu and Xinyang Lu and Yifei Li and Jie Fu and Junxian He and Bryan Hooi},
booktitle={The Twelfth International Conference on Learning Representations},
year={2024},
url={https://openreview.net/forum?id=gjeQKFxFpZ}
}

@misc{wan20252025foundationmodeltransparency,
      title={The 2025 Foundation Model Transparency Index}, 
      author={Alexander Wan and Kevin Klyman and Sayash Kapoor and Nestor Maslej and Shayne Longpre and Betty Xiong and Percy Liang and Rishi Bommasani},
      year={2025},
      eprint={2512.10169},
      archivePrefix={arXiv},
      primaryClass={cs.AI},
      url={https://arxiv.org/abs/2512.10169}, 
}

@InProceedings{pmlr-v286-xiao25a,
  title = {The Consistency Hypothesis in Uncertainty Quantification for Large Language Models},
  author = {Xiao, Quan and Bhattacharjya, Debarun and Ganesan, Balaji and Marinescu, Radu and Mirylenka, Katya and Pham, Nhan H and Glass, Michael and Lee, Junkyu},
  booktitle = {Proceedings of the Forty-first Conference on Uncertainty in Artificial Intelligence},
  pages =  {4636--4651},
  year = {2025},
  editor = {Chiappa, Silvia and Magliacane, Sara},
  volume = {286},
  series = {Proceedings of Machine Learning Research},
  month = {21--25 Jul},
  publisher = {PMLR},
  url = {https://proceedings.mlr.press/v286/xiao25a.html},
 }

@article{pedapati2024large,
  title={Large language model confidence estimation via black-box access},
  author={Pedapati, Tejaswini and Dhurandhar, Amit and Ghosh, Soumya and Dan, Soham and Sattigeri, Prasanna},
  journal={arXiv preprint arXiv:2406.04370},
  year={2024}
}

@inproceedings{tao2025revisiting,
title={Revisiting Uncertainty Estimation and Calibration of Large Language Models},
author={Linwei Tao and Yi-Fan Yeh and Minjing Dong and Tao Huang and Jialin Yu and Philip Torr and Chang Xu},
booktitle={Workshop on Scaling Environments for Agents},
year={2025},
url={https://openreview.net/forum?id=Q9CreVjHH7}
}

@article{
lin2022teaching,
title={Teaching Models to Express Their Uncertainty in Words},
author={Stephanie Lin and Jacob Hilton and Owain Evans},
journal={Transactions on Machine Learning Research},
issn={2835-8856},
year={2022},
url={https://openreview.net/forum?id=8s8K2UZGTZ},
note={}
}

@article{WANG20251612,
    author = {Wang, Hongchen and Li, Kangming and Ramsay, Scott and Fehlis, Yao and Kim, Edward and Hattrick-Simpers, Jason},
    title = {Evaluating the performance and robustness of LLMs in materials science Q\&amp;A and property predictions},
    journal = {Digital Discovery},
    volume = {4},
    number = {6},
    pages = {1612-1624},
    year = {2025},
    month = {06},

    issn = {2635-098X},
    doi = {10.1039/d5dd00090d},
    url = {https://doi.org/10.1039/d5dd00090d},
    eprint = {https://pubs.rsc.org/dd/article-pdf/4/6/1612/10331509/d5dd00090d.pdf},
}

@article{yang2024just,
  title={Just rephrase it! Uncertainty estimation in closed-source language models via multiple rephrased queries},
  author={Yang, Adam and Chen, Chen and Pitas, Konstantinos},
  journal={arXiv preprint arXiv:2405.13907},
  year={2024}
}

@article{zeng2024uncertainty,
  title={Uncertainty is fragile: Manipulating uncertainty in large language models},
  author={Zeng, Qingcheng and Jin, Mingyu and Yu, Qinkai and Wang, Zhenting and Hua, Wenyue and Zhou, Zihao and Sun, Guangyan and Meng, Yanda and Ma, Shiqing and Wang, Qifan and others},
  journal={arXiv preprint arXiv:2407.11282},
  year={2024}
}

@article{obadinma2025robustness,
  title={On the Robustness of Verbal Confidence of {LLM}s in Adversarial Attacks},
  author={Obadinma, Stephen and Zhu, Xiaodan},
  journal={arXiv preprint arXiv:2507.06489},
  year={2025}
}

@article{obadinma2024calibration,
  title={Calibration attacks: A comprehensive study of adversarial attacks on model confidence},
  author={Obadinma, Stephen and Zhu, Xiaodan and Guo, Hongyu},
  journal={arXiv preprint arXiv:2401.02718},
  year={2024}
}

@misc{shen2025safetyinstinctsLLMslearn,
      title={Safety Instincts: {LLM}s Learn to Trust Their Internal Compass for Self-Defense}, 
      author={Guobin Shen and Dongcheng Zhao and Haibo Tong and Jindong Li and Feifei Zhao and Yi Zeng},
      year={2025},
      eprint={2510.01088},
      archivePrefix={arXiv},
      primaryClass={cs.AI},
      url={https://arxiv.org/abs/2510.01088}, 
}

@misc{epstein2025LLMsoverconfidentevaluatingconfidence,
      title={{LLM}s are Overconfident: Evaluating Confidence Interval Calibration with FermiEval}, 
      author={Elliot L. Epstein and John Winnicki and Thanawat Sornwanee and Rajat Dwaraknath},
      year={2025},
      eprint={2510.26995},
      archivePrefix={arXiv},
      primaryClass={stat.ME},
      url={https://arxiv.org/abs/2510.26995}, 
}

@article{XU2025101370,
title = {Toward large reasoning models: A survey of reinforced reasoning with large language models},
journal = {Patterns},
volume = {6},
number = {10},
pages = {101370},
year = {2025},
issn = {2666-3899},
doi = {https://doi.org/10.1016/j.patter.2025.101370},
url = {https://www.sciencedirect.com/science/article/pii/S2666389925002181},
author = {Fengli Xu and Qianyue Hao and Chenyang Shao and Zefang Zong and Yu Li and Jingwei Wang and Yunke Zhang and Jingyi Wang and Xiaochong Lan and Jiahui Gong and Tianjian Ouyang and Fanjin Meng and Yuwei Yan and Qinglong Yang and Yiwen Song and Sijian Ren and Xinyuan Hu and Jie Feng and Chen Gao and Yong Li}
}

@misc{openai2025gptoss120bgptoss20bmodel,
      title={gpt-oss-120b \& gpt-oss-20b Model Card}, 
      author={OpenAI},
      year={2025},
      eprint={2508.10925},
      archivePrefix={arXiv},
      primaryClass={cs.CL},
      url={https://arxiv.org/abs/2508.10925}, 
}

@misc{yang2025qwen3technicalreport,
      title={Qwen3 Technical Report}, 
      author={{Qwen Team}},
      year={2025},
      eprint={2505.09388},
      archivePrefix={arXiv},
      primaryClass={cs.CL},
      url={https://arxiv.org/abs/2505.09388}, 
}

@misc{openai2024gpt4technicalreport,
      title={GPT-4 Technical Report}, 
      author={OpenAI},
      year={2024},
      eprint={2303.08774},
      archivePrefix={arXiv},
      primaryClass={cs.CL},
      url={https://arxiv.org/abs/2303.08774}, 
}

@article{qwen3embedding,
  title={Qwen3 Embedding: Advancing Text Embedding and Reranking Through Foundation Models},
  author={Zhang, Yanzhao and Li, Mingxin and Long, Dingkun and Zhang, Xin and Lin, Huan and Yang, Baosong and Xie, Pengjun and Yang, An and Liu, Dayiheng and Lin, Junyang and Huang, Fei and Zhou, Jingren},
  journal={arXiv preprint arXiv:2506.05176},
  year={2025}
}

@inproceedings{ICLR2025_29fb6e14,
 author = {Leng, Jixuan and Huang, Chengsong and Zhu, Banghua and Huang, Jiaxin},
 booktitle = {International Conference on Learning Representations},
 editor = {Y. Yue and A. Garg and N. Peng and F. Sha and R. Yu},
 pages = {16484--16517},
 title = {Taming Overconfidence in {LLM}s: Reward Calibration in {RLHF}},
 url = {https://proceedings.iclr.cc/paper_files/paper/2025/file/29fb6e1456b3d8b57ede5c45aa2c6537-Paper-Conference.pdf},
 volume = {2025},
 year = {2025}
}

@article{deepseek,
   title={DeepSeek-R1 incentivizes reasoning in {LLM}s through reinforcement learning},
   volume={645},
   ISSN={1476-4687},
   url={http://dx.doi.org/10.1038/s41586-025-09422-z},
   DOI={10.1038/s41586-025-09422-z},
   number={8081},
   journal={Nature},
   publisher={Springer Science and Business Media LLC},
   author={DeepSeek-AI},
   year={2025},
   month=sep, pages={633–638} }

@inproceedings{10.5555/3666122.3669630,
author = {Wei, Alexander and Haghtalab, Nika and Steinhardt, Jacob},
title = {Jailbroken: how does {LLM} safety training fail?},
year = {2023},
publisher = {Curran Associates Inc.},
address = {Red Hook, NY, USA},
booktitle = {Proceedings of the 37th International Conference on Neural Information Processing Systems},
articleno = {3508},
numpages = {32},
location = {New Orleans, LA, USA},
series = {NIPS '23}
}

@article{ouyang2022training,
  title={Training language models to follow instructions with human feedback},
  author={Ouyang, Long and Wu, Jeffrey and Jiang, Xu and Almeida, Diogo and Wainwright, Carroll and Mishkin, Pamela and Zhang, Chong and Agarwal, Sandhini and Slama, Katarina and Ray, Alex and others},
  journal={Advances in neural information processing systems},
  volume={35},
  pages={27730--27744},
  year={2022}
}

@article{rafailov2023direct,
  title={Direct preference optimization: Your language model is secretly a reward model},
  author={Rafailov, Rafael and Sharma, Archit and Mitchell, Eric and Manning, Christopher D and Ermon, Stefano and Finn, Chelsea},
  journal={Advances in neural information processing systems},
  volume={36},
  pages={53728--53741},
  year={2023}
}

@article{shao2024deepseekmath,
  title={Deepseekmath: Pushing the limits of mathematical reasoning in open language models},
  author={Shao, Zhihong and Wang, Peiyi and Zhu, Qihao and Xu, Runxin and Song, Junxiao and Bi, Xiao and Zhang, Haowei and Zhang, Mingchuan and Li, YK and Wu, Yang and others},
  journal={arXiv preprint arXiv:2402.03300},
  year={2024}
}

@article{zheng2025group,
  title={Group sequence policy optimization},
  author={Zheng, Chujie and Liu, Shixuan and Li, Mingze and Chen, Xiong-Hui and Yu, Bowen and Gao, Chang and Dang, Kai and Liu, Yuqiong and Men, Rui and Yang, An and others},
  journal={arXiv preprint arXiv:2507.18071},
  year={2025}
}

@article{schulman2017proximal,
  title={Proximal policy optimization algorithms},
  author={Schulman, John and Wolski, Filip and Dhariwal, Prafulla and Radford, Alec and Klimov, Oleg},
  journal={arXiv preprint arXiv:1707.06347},
  year={2017}
}

@inproceedings{
gx-chen2026klregularized,
title={{KL}-Regularized Reinforcement Learning is Designed to Mode Collapse},
author={Anthony GX-Chen and Jatin Prakash and Jeff Guo and Rob Fergus and Rajesh Ranganath},
booktitle={The Fourteenth International Conference on Learning Representations},
year={2026},
url={https://openreview.net/forum?id=flBRtdIihA}
}

@inproceedings{azar2024general,
  title={A general theoretical paradigm to understand learning from human preferences},
  author={Azar, Mohammad Gheshlaghi and Guo, Zhaohan Daniel and Piot, Bilal and Munos, Remi and Rowland, Mark and Valko, Michal and Calandriello, Daniele},
  booktitle={International Conference on Artificial Intelligence and Statistics},
  pages={4447--4455},
  year={2024},
  organization={PMLR}
}

@inproceedings{
wang2023selfconsistency,
title={Self-Consistency Improves Chain of Thought Reasoning in Language Models},
author={Xuezhi Wang and Jason Wei and Dale Schuurmans and Quoc V Le and Ed H. Chi and Sharan Narang and Aakanksha Chowdhery and Denny Zhou},
booktitle={The Eleventh International Conference on Learning Representations },
year={2023},
url={https://openreview.net/forum?id=1PL1NIMMrw}
}

@inproceedings{
kuhn2023semantic,
title={Semantic Uncertainty: Linguistic Invariances for Uncertainty Estimation in Natural Language Generation},
author={Lorenz Kuhn and Yarin Gal and Sebastian Farquhar},
booktitle={The Eleventh International Conference on Learning Representations },
year={2023},
url={https://openreview.net/forum?id=VD-AYtP0dve}
}

@inproceedings{naeini2015obtaining,
  title={Obtaining well calibrated probabilities using bayesian binning},
  author={Naeini, Mahdi Pakdaman and Cooper, Gregory and Hauskrecht, Milos},
  booktitle={Proceedings of the AAAI conference on artificial intelligence},
  volume={29},
  number={1},
  year={2015}
}

@article{glenn1950verification,
  title={Verification of forecasts expressed in terms of probability},
  author={{Brier}, Glenn W.},
  journal={Monthly weather review},
  volume={78},
  number={1},
  pages={1--3},
  year={1950},
  publisher={War Department, Office of the Chief Signal Officer}
}

@article{cobbe2021gsm8k,
  title={Training Verifiers to Solve Math Word Problems},
  author={Cobbe, Karl and Kosaraju, Vineet and Bavarian, Mohammad and Chen, Mark and Jun, Heewoo and Kaiser, Lukasz and Plappert, Matthias and Tworek, Jerry and Hilton, Jacob and Nakano, Reiichiro and Hesse, Christopher and Schulman, John},
  journal={arXiv preprint arXiv:2110.14168},
  year={2021}
}

@inproceedings{guo2017calibration,
author = {Guo, Chuan and Pleiss, Geoff and Sun, Yu and Weinberger, Kilian Q.},
title = {On calibration of modern neural networks},
year = {2017},
publisher = {JMLR.org},

booktitle = {Proceedings of the 34th International Conference on Machine Learning - Volume 70},
pages = {1321–1330},
numpages = {10},
location = {Sydney, NSW, Australia},
series = {ICML'17}
}

@article{zhou2023lima,
  title={Lima: Less is more for alignment},
  author={Zhou, Chunting and Liu, Pengfei and Xu, Puxin and Iyer, Srinivasan and Sun, Jiao and Mao, Yuning and Ma, Xuezhe and Efrat, Avia and Yu, Ping and Yu, Lili and others},
  journal={Advances in Neural Information Processing Systems},
  volume={36},
  pages={55006--55021},
  year={2023}
}

@inproceedings{yukhymenko2024synthetic,
    title={A Synthetic Dataset for Personal Attribute Inference}, 
    author={Hanna Yukhymenko and Robin Staab and Mark Vero and Martin Vechev},
    year={2024},
    booktitle={Thirty-eighth Conference on Neural Information Processing Systems Datasets and Benchmarks Track},
    url={https://openreview.net/forum?id=1nqfIQIQBf}
}

@inproceedings{staab24beyond,
    title={Beyond Memorization: Violating Privacy via Inference with Large Language Models},
    author={Robin Staab and Mark Vero and Mislav Balunović and Martin Vechev},
    booktitle={The Twelfth International Conference on Learning Representations},
    year={2024},
}

@inproceedings{
du2025supergpqa,
title={Super{GPQA}: Scaling {{LLM}} Evaluation across 285 Graduate Disciplines},
author={Xeron Du and Yifan Yao and Kaijing Ma and {et al.}},
booktitle={The Thirty-ninth Annual Conference on Neural Information Processing Systems Datasets and Benchmarks Track},
year={2025},
url={https://openreview.net/forum?id=6WgflzYQpf}
}

@misc{hle,
      title={Humanity's Last Exam}, 
      author={Long Phan and Alice Gatti and Ziwen Han and {et al.}},
      year={2025},
      eprint={2501.14249},
      archivePrefix={arXiv},
      primaryClass={cs.LG},
      url={https://arxiv.org/abs/2501.14249}, 
}

@misc{zou2023universaltransferableadversarialattacks,
      title={Universal and Transferable Adversarial Attacks on Aligned Language Models}, 
      author={Andy Zou and Zifan Wang and Nicholas Carlini and Milad Nasr and J. Zico Kolter and Matt Fredrikson},
      year={2023},
      eprint={2307.15043},
      archivePrefix={arXiv},
      primaryClass={cs.CL},
      url={https://arxiv.org/abs/2307.15043}, 
}

@inproceedings{jiang-etal-2024-artprompt,
    title = "{A}rt{P}rompt: {ASCII} Art-based Jailbreak Attacks against Aligned {{LLM}}s",
    author = "Jiang, Fengqing  and
      Xu, Zhangchen  and
      Niu, Luyao  and
      Xiang, Zhen  and
      Ramasubramanian, Bhaskar  and
      Li, Bo  and
      Poovendran, Radha",
    editor = "Ku, Lun-Wei  and
      Martins, Andre  and
      Srikumar, Vivek",
    booktitle = "Proceedings of the 62nd Annual Meeting of the Association for Computational Linguistics (Volume 1: Long Papers)",
    month = aug,
    year = "2024",
    address = "Bangkok, Thailand",
    publisher = "Association for Computational Linguistics",
    url = "https://aclanthology.org/2024.acl-long.809/",
    doi = "10.18653/v1/2024.acl-long.809",
    pages = "15157--15173",

}

@article{fisher1922mathematical,
  title={On the mathematical foundations of theoretical statistics},
  author={Fisher, Ronald A},
  journal={Philosophical transactions of the Royal Society of London. Series A, containing papers of a mathematical or physical character},
  volume={222},
  number={594-604},
  pages={309--368},
  year={1922},
  publisher={The Royal Society London}
}

@misc{yi2024jailbreakattacksdefenseslarge,
      title={Jailbreak Attacks and Defenses Against Large Language Models: A Survey}, 
      author={Sibo Yi and Yule Liu and Zhen Sun and Tianshuo Cong and Xinlei He and Jiaxing Song and Ke Xu and Qi Li},
      year={2024},
      eprint={2407.04295},
      archivePrefix={arXiv},
      primaryClass={cs.CR},
      url={https://arxiv.org/abs/2407.04295}, 
}

@inproceedings{DBLP:conf/acl/XuLDLP24,
  author       = {Zihao Xu and
                  Yi Liu and
                  Gelei Deng and
                  Yuekang Li and
                  Stjepan Picek},
  editor       = {Lun{-}Wei Ku and
                  Andre Martins and
                  Vivek Srikumar},
  title        = {A Comprehensive Study of Jailbreak Attack versus Defense for Large
                  Language Models},
  booktitle    = {Findings of the Association for Computational Linguistics, {ACL} 2024,
                  Bangkok, Thailand and virtual meeting, August 11-16, 2024},
  series       = {Findings of {ACL}},
  volume       = {{ACL} 2024},
  pages        = {7432--7449},
  publisher    = {Association for Computational Linguistics},
  year         = {2024},
  url          = {https://doi.org/10.18653/v1/2024.findings-acl.443},
  doi          = {10.18653/V1/2024.FINDINGS-ACL.443},
  bibsource    = {dblp computer science bibliography, https://dblp.org}
}

@misc{pathade2025redteamingmindmachine,
      title={Red Teaming the Mind of the Machine: A Systematic Evaluation of Prompt Injection and Jailbreak Vulnerabilities in {LLM}s}, 
      author={Chetan Pathade},
      year={2025},
      eprint={2505.04806},
      archivePrefix={arXiv},
      primaryClass={cs.CR},
      url={https://arxiv.org/abs/2505.04806}, 
}

@misc{mccann2018naturallanguagedecathlonmultitask,
      title={The Natural Language Decathlon: Multitask Learning as Question Answering}, 
      author={Bryan McCann and Nitish Shirish Keskar and Caiming Xiong and Richard Socher},
      year={2018},
      eprint={1806.08730},
      archivePrefix={arXiv},
      primaryClass={cs.CL},
      url={https://arxiv.org/abs/1806.08730}, 
}

@inproceedings{khashabi-etal-2020-unifiedqa,
    title = "{UNIFIEDQA}: Crossing Format Boundaries with a Single {QA} System",
    author = "Khashabi, Daniel  and
      Min, Sewon  and
      Khot, Tushar  and
      Sabharwal, Ashish  and
      Tafjord, Oyvind  and
      Clark, Peter  and
      Hajishirzi, Hannaneh",
    editor = "Cohn, Trevor  and
      He, Yulan  and
      Liu, Yang",
    booktitle = "Findings of the Association for Computational Linguistics: EMNLP 2020",
    month = nov,
    year = "2020",
    address = "Online",
    publisher = "Association for Computational Linguistics",
    url = "https://aclanthology.org/2020.findings-emnlp.171/",
    doi = "10.18653/v1/2020.findings-emnlp.171",
    pages = "1896--1907"
}

@article{10.1561/2200000101,
author = {Angelopoulos, Anastasios N. and Bates, Stephen},
title = {Conformal Prediction: A Gentle Introduction},
year = {2023},
issue_date = {Mar 2023},
publisher = {Now Publishers Inc.},
address = {Hanover, MA, USA},
volume = {16},
number = {4},
issn = {1935-8237},
url = {https://doi.org/10.1561/2200000101},
doi = {10.1561/2200000101},
journal = {Found. Trends Mach. Learn.},
month = mar,
pages = {494–591},
numpages = {114}
}

@inproceedings{
andriushchenko2025does,
title={Does Refusal Training in {{LLM}}s Generalize to the Past Tense?},
author={Maksym Andriushchenko and Nicolas Flammarion},
booktitle={The Thirteenth International Conference on Learning Representations},
year={2025},
url={https://openreview.net/forum?id=aJUuere4fM}
}

@inproceedings{DBLP:conf/naacl/DingKMCXCH24,
  author       = {Peng Ding and
                  Jun Kuang and
                  Dan Ma and
                  Xuezhi Cao and
                  Yunsen Xian and
                  Jiajun Chen and
                  Shujian Huang},
  editor       = {Kevin Duh and
                  Helena G{\'{o}}mez{-}Adorno and
                  Steven Bethard},
  title        = {A Wolf in Sheep's Clothing: Generalized Nested Jailbreak Prompts can
                  Fool Large Language Models Easily},
  booktitle    = {Proceedings of the 2024 Conference of the North American Chapter of
                  the Association for Computational Linguistics: Human Language Technologies
                  (Volume 1: Long Papers), {NAACL} 2024, Mexico City, Mexico, June 16-21,
                  2024},
  pages        = {2136--2153},
  publisher    = {Association for Computational Linguistics},
  year         = {2024},
  url          = {https://doi.org/10.18653/v1/2024.naacl-long.118},
  doi          = {10.18653/V1/2024.NAACL-LONG.118},
  bibsource    = {dblp computer science bibliography, https://dblp.org}
}

@book{boucheron2013inequalities,
    author = {Boucheron, Stéphane and Lugosi, Gábor and Massart, Pascal},
    title = {Concentration Inequalities: A Nonasymptotic Theory of Independence},
    publisher = {Oxford University Press},
    year = {2013},
    month = {02},
    isbn = {9780199535255},
    doi = {10.1093/acprof:oso/9780199535255.001.0001},
    url = {https://doi.org/10.1093/acprof:oso/9780199535255.001.0001},
}

@misc{openai2026gpt56,
  author       = {{OpenAI}},
  title        = {{GPT-5.6: Frontier Intelligence That Scales with Your Ambition}},
  year         = {2026},
  month        = jul,
  howpublished = {\url{https://openai.com/index/gpt-5-6/}},
}
\bibliographystyle{iclr2025_conference}

\appendix

\section{Proofs}
\label{app:all_proofs}
\subsection{Proof of the proposition \ref{prop:KL_divergence}} \label{app:prop_proof}

\begin{proof}
Let $\pi_\beta$ be the policy defined in Eq.~\ref{eq:gibbs_policy}. 
By the definition of the relaxation operator $j \in \mathcal{J}$, there exist $\beta_1, \beta_2$ such that $\pi_j \equiv \pi_{\beta_1}$ and $\pi_{\text{RL}} \equiv \pi_{\beta_2}$ with $\beta_1 > \beta_2$.
Consider the mapping $D: \mathbb{R}^*_+ \to \mathbb{R}_+$ defined by $D(\beta) = \mathbb{D}_{\text{KL}}(\pi_\beta \parallel \pi_{\text{ref}})$.
Expanding the KL divergence using the closed-form expression of $\pi_\beta$:
$$
D(\beta) = \mathbb{E}_{t \sim \pi_\beta} \left[ \log \frac{\pi_\beta(t|x)}{\pi_{\text{ref}}(t|x)} \right] = \mathbb{E}_{t \sim \pi_\beta} \left[ \frac{R(x,t)}{\beta} - \log Z_\beta(x) \right]
$$
We evaluate the monotonicity of $D(\beta)$ by computing its derivative with respect to $\beta$:
\begin{enumerate}
    \item Derivative of the log-partition function: $$ \frac{\partial}{\partial \beta} \log Z_\beta(x) = \frac{1}{Z_\beta(x)} \sum \pi_{\text{ref}}(t|x) e^{R/\beta} \left( -\frac{R}{\beta^2} \right)  = - \frac{1}{\beta^2} \mathbb{E}_{\pi_\beta}[R(x,t)] $$.
    \item Derivative of the expectation term: Assuming $R$ is bounded, $D(\beta)$ is differentiable under the integral sign. We can use the identity $\nabla_\beta \mathbb{E}_{\pi_\beta}[f] = \mathbb{E}_{\pi_\beta}[f \nabla_\beta \log \pi_\beta]$ and noting that $\nabla_\beta \log \pi_\beta = -\frac{1}{\beta^2}(R - \mathbb{E}_{\pi_\beta}[R])$, we obtain:
\end{enumerate}
$$
\frac{\partial}{\partial \beta} \left( \frac{1}{\beta} \mathbb{E}_{\pi_\beta}[R] \right) = -\frac{1}{\beta^2} \mathbb{E}_{\pi_\beta}[R] - \frac{1}{\beta^3} \text{Var}_{\pi_\beta}(R(x,t))
$$
Combining these results:
$$
\frac{d D(\beta)}{d \beta} = \left[ -\frac{1}{\beta^2} \mathbb{E}_{\pi_\beta}[R] - \frac{1}{\beta^3} \text{Var}_{\pi_\beta}(R) \right] - \left[ - \frac{1}{\beta^2} \mathbb{E}_{\pi_\beta}[R] \right] = - \frac{\text{Var}_{\pi_\beta}(R(x,t))}{\beta^3}
$$
Since $\text{Var}_{\pi_\beta}(R(x,t)) \geq 0$ and $\beta > 0$, then $\frac{d D(\beta)}{d \beta} \leq 0$. Assuming $R(x,t)$ is non-constant on the support of $\pi_{\text{ref}}$, the variance is strictly positive, making $D(\beta)$ strictly decreasing. Given $\beta_1 > \beta_2$, it follows that $D(\beta_1) < D(\beta_2)$, which completes the proof.
\end{proof}

\subsection{Proof of the theorem \ref{thm:relaxation_ece}}\label{app:thm_proof}

\begin{proof}

The proof works in three parts.
First, we decompose ECE in an over-aligned setting as an expected value of confidence minus an accuracy.
Second, we show that higher $\beta$ parameters lead to lower expected confidence.
Finally, we conclude that relaxation operators that augment $\beta$ consequently lower ECE.

For the first point, we remind that for a policy $\pi$,
we defined the majority response estimator on $K$ sample of size $\hat{Y}_\pi^{(K)}(x)$ (see Def.~\ref{def:majority_estimator})
that converges towards the theoretical mode $y^*_\pi$ for larger sample sizes (see Def.~\ref{def:theoretical_mode}),
and that the estimated confidence $\hat c_\pi$ is the frequency of this majority response (see Def. ~\ref{def:confidence_estimator}).
We also defined accuracy indicator (see Def. \ref{def:accuracy_indicator}) as $I_\pi \coloneqq \indicatrice{\hat{Y}^{(K)}_\pi(X) = Y}$ which is $1$ when $\hat{Y}_\pi^{(K)}(x)$ is correct and $0$ otherwise.

In this setting, we recall our assumptions:
\begin{itemize}
    \item The policy we consider is a closed-form optimal policy (see Eq.~\ref{eq:gibbs_policy}).
    \item The policy we consider is overconfident (see Def.~\ref{def:overconfidence}).
    \item The policy is the result of an excessive RL-based alignment. This notably implies that:
    \begin{itemize}

        \item the reasonings leading to the theoretical mode $y^*_\pi$ have better-than-average rewards.

        \item there exist relaxations that are not strong enough to change the theoretical mode and that keep an overconfident regime.

    \end{itemize}
\end{itemize}

\paragraph{Why aggressive alignment yields a reward-maximizing, dominant mode.}
Grouping traces by answer, Eq.~\ref{eq:gibbs_policy} gives
$\pi_\beta(y \mid x) \propto \sum_{t:\, f(t)=y} \pi_{\mathrm{ref}}(t \mid x)\, e^{R(x,t)/\beta}$.
When alignment is aggressive, i.e.\ $\beta$ is small relative to reward differences, each sum is
dominated by its highest-reward traces, so for two answers $y_1, y_2$ with best rewards
$R^*_{y_1} > R^*_{y_2}$,
\[
  \frac{\pi_\beta(y_1 \mid x)}{\pi_\beta(y_2 \mid x)} \;\approx\;
  \frac{\pi_{\mathrm{ref}}(\mathcal{T}^*_{y_1} \mid x)}{\pi_{\mathrm{ref}}(\mathcal{T}^*_{y_2} \mid x)}
  \, \exp\!\Big(\frac{R^*_{y_1} - R^*_{y_2}}{\beta}\Big),
\]
where $\mathcal{T}^*_{y}$ denotes the highest-reward traces leading to $y$. The reward gap is
amplified exponentially in $1/\beta$, so the answer reached by the highest-reward reasoning
becomes the theoretical mode $y^*$, and it dominates its closest competitor by a large factor.
Two consequences follow. First, since almost all of the mass of $y^*$ sits on near-maximal-reward
traces, the reasonings leading to $y^*$ have better-than-average reward. Second, the mode is robust
to moderate increases of $\beta$: it can only be overtaken once the exponential factor no longer
compensates the reference-mass ratio, i.e.\ roughly when
$\beta \gtrsim (R^*_{y^*} - R^*_{y_2}) / \log\big(\pi_{\mathrm{ref}}(\mathcal{T}^*_{y_2}\mid x) /
\pi_{\mathrm{ref}}(\mathcal{T}^*_{y^*}\mid x)\big)$, and never if $y^*$ also carries more reference
mass. This leaves a range of relaxations that lower confidence without changing the mode.

\paragraph{ECE as Expected value minus Accuracy.}

The ECE is best understood as a measure of the tension between the model's confidence (here estimated by $\hat{c}_\pi$), and its accuracy.
When the model is overconfident, this tension only ever goes one way,
as the majority response is always more frequent than it needs to be.

\begin{align*}
\text{ECE}(\pi)
&=
\mathbb{E} \left[ \big| \mathbb{P}(I_\pi = 1 \mid \hat{c}_\pi) - \hat{c}_\pi \big| \right] \\
&=
\mathbb{E} \left[\hat{c}_\pi- \mathbb{P}(I_\pi = 1 \mid \hat{c}_\pi)  \right] &\text{by assumption of overconfidence} \\
&=
\mathbb{E} \left[\hat{c}_\pi  \right]  - \mathbb{E} \left[ \mathbb{P}(I_\pi = 1 \mid \hat{c}_\pi)  \right] \\
&=
\mathbb{E} \left[\hat{c}_\pi  \right]  -  \mathbb{P}(I_\pi = 1) &  \text{by law of total probability} \\
\end{align*}

Furthermore, for all realizations $y_1, \dots, y_K$ of $f(T_1), \dots, f(T_K)$, we have:

$$
\hat{c}_\pi(x)
\coloneqq
\frac{1}{K} \sum_{k=1}^K \indicatrice{f(T_k)=\hat{Y}_\pi^{(K)} } \\
=
\max_{y \in \mathcal Y}\frac{1}{K} \sum_{k=1}^K \indicatrice{y_k = y} \\
\coloneqq
\max_{y \in \mathcal Y}\hat{p}_y
$$

Since expected value of a maximum is greater than the maximum of the expected values, we have:

$$
\mathbb{E}\left[ \max_{y \in \mathcal{Y}} \hat{p}_y \right] \geq \max_{y \in \mathcal{Y}} \mathbb{E}[\hat{p}_y] = \max_{y \in \mathcal{Y}} p_y = p_{y^*_\pi} \coloneqq p^*
$$

We will rephrase this inequality as a sum:

$$
\mathbb{E}[\hat{c}_\pi(x)] = p^* + \text{bias}_K(x), \quad \text{with } \text{bias}_K(x) \geq 0
$$

Since $\hat c_\pi$ is bounded, it is subgaussian.
As such we can use Hoeffding's inequality to put an upper bound on this bias.
For all $\varepsilon >0$,
$$
\mathbb{P}\left( \max_{y \in \mathcal{Y}} \hat{p}_y \geq p^* + \varepsilon \right) \leq |\mathcal{Y}| \exp(-2K\varepsilon^2)
$$
By a standard maximal inequality for sub-Gaussian variables (see, e.g., \cite{boucheron2013inequalities}), we obtain the following bound on the bias:
$$
\text{bias}_K(x) \leq \int_0^1 \min\left(1, |\mathcal{Y}|e^{-2K\varepsilon^2}\right) d\varepsilon \lesssim \sqrt{\frac{\log |\mathcal{Y}|}{K}}
$$

From this bias estimation we get the following on expected confidence:

$$
\mathbb{E}[\hat{c}_\pi(x)] = p^* + \mathcal{O}\left( \sqrt{\frac{\log |\mathcal{Y}|}{K}} \right)
$$

If we combine this formula with the equation obtained for ECE under the overconfidence assumption, we get:

$$
\text{ECE}(\pi) = p^* -  \mathbb{P}(I_\pi = 1) + \mathcal{O}\left( \sqrt{\frac{\log |\mathcal{Y}|}{K}} \right) 
$$

\paragraph{Optimal policy is monotonous when $\beta$ varies.}\label{proof:monotonie}

By assumption we are only considering policies $\{\pi_\beta\}_{\beta>0}$ as defined in \eqref{eq:gibbs_policy}.

As a reminder, since the reward function $R$ is assumed bounded, we have that $\beta\mapsto \pi_\beta(t\mid x)$ is $C^1$ at every $(x,t)\in \mathcal X \times \mathcal T$. As such the derivative of its expected value is the expected value of its derivative.

Our assumption that the original behaviour is overaligned by RL covers some margin of relaxation for which the theoretical mode is unchanged, as the theoretical mode is bolstered by strong rewards. As such, we are showing our theorem by varying $\beta$ in this span.

We compute $\partial_\beta \log \pi_\beta(t|x)$ the partial derivative of $\pi_\beta(t|x)$ with respect to $\beta$:
$$
\partial_\beta \log \pi_\beta(t|x) = -\frac{R(x,t)}{\beta^2}-\partial_\beta \log Z_\beta(x)
$$
However, we know of the derivative of $\log Z_\beta$ with respect to $\beta$ that:
\begin{align*}
\partial_\beta \log Z_\beta(x)
= \frac{1}{Z_\beta(x)} \sum_{t'} \pi_{\text{ref}}(t'|x) \exp \left(\frac{R(x,t')}{\beta}\right) \left(-\frac{R(x,t')}{\beta^2} \right)
= - \frac{1}{\beta^2} \sum_{t'} \pi_\beta(t'|x)R(x,t')
= - \frac{1}{\beta^2} \mathbb E [R(x,t')]
\end{align*}
We replace this in the previous equation to get:
\begin{align*}
\partial_\beta \log \pi_\beta(t|x)
= -\frac{R(x,t)}{\beta^2} + \frac{1}{\beta^2} \mathbb E[R(x,t')]
= \frac{1}{\beta^2} ( \mathbb E[R(x,t')]- R(x,t))
\end{align*}
Finally, since $\partial_\beta \pi_\beta(t|x) = \pi_\beta(t|x) \partial_\beta \log \pi_\beta(t|x)$, we get:
$$
\partial_\beta \pi_\beta(t|x) = \frac{\pi_\beta(t|x)}{\beta^2} \big( \mathbb{E}[R(x,t')] - R(x,t) \big)
$$
The marginal probability of the final output being $y$ is the sum of its probabilities over the space of all reasonings $\mathcal{T}_y = \{t \in \mathcal{T} \mid f(t) = y\}$.
The derivative with respect to $\beta$ of this sum is:
\begin{align*}
\partial_\beta \pi_\beta(y|x)
&=
\sum_{t \in \mathcal{T}_y} \frac{\pi_\beta(t|x)}{\beta^2} \big( \mathbb{E}[R(x,t')] - R(x,t) \big) \\
&=
\frac{\pi_\beta(y|x)}{\beta^2} \left( \mathbb{E}[R(x,t')] - \sum_{t \in \mathcal{T}_y} \frac{\pi_\beta(t|x)}{\pi_\beta(y|x)} R(x,t) \right) \\
&=
\frac{\pi_\beta(y|x)}{\beta^2} \big( \mathbb{E}[R(x,t')] - \mathbb{E}[R(x,t) \mid f(t)=y] \big)
\end{align*}
For the theoretical mode $y^*_\pi$ reasoning traces leading to $y^*_\pi$ have on average a better-than-average reward, that is to say $\mathbb{E}[R(x,t) \mid f(t)=y^*_\pi] \geq \mathbb{E}[R(x,t')]$. On that condition, $\partial_\beta \pi_\beta(y^*_\pi|x) \leq 0$.

Thus $p^*$ decreases as $\beta$ grows.

For any $\beta_1 > \beta_2$, we can compare the Expected Calibration Errors (ECE) as follows:
\begin{equation}
\text{ECE}(\pi_{\beta_1}) - \text{ECE}(\pi_{\beta_2}) = (p^*_{\beta_1} - p^*_{\beta_2}) - \left( P(I_{\pi_{\beta_1}} = 1) - P(I_{\pi_{\beta_2}} = 1) \right) + O\left( \sqrt{\frac{\log |\mathcal{Y}|}{K}} \right)
\end{equation}
Let $\Delta \text{Acc} = P(I_{\pi_{\beta_1}} = 1) - P(I_{\pi_{\beta_2}} = 1)$ and $\Delta p^* = p^*_{\beta_1} - p^*_{\beta_2}$.

\paragraph{Confidence Variation}
From the derivation of the Gibbs policy presented in the previous section, we have for all $x$:
\begin{equation}
\frac{\partial p^*_\beta(x)}{\partial \beta} = - \frac{p^*_\beta(x)\,\delta_R(x,\beta)}{\beta^2}
\end{equation}
where $\delta_R(x,\beta) = \mathbb{E}[R \mid y^*] - \mathbb{E}[R] > 0$ in the overconfidence regime.

Consequently, $p^*_\beta(x)$ is strictly decreasing with respect to $\beta$, and its variation is of the first order:
\begin{equation}
|\Delta p^*| \;\gtrsim\; \int_{\beta_2}^{\beta_1} \frac{c(x)}{\beta^2} d\beta,
\end{equation}
for a positive bounding function $c(x) > 0$.

\paragraph{Accuracy Variation}
By definition, the accuracy is given by:
\begin{equation}
\text{Acc}(\beta) = P(\hat Y^{(K)}_\pi = Y \mid X=x).
\end{equation}
The error of the majority vote stems from sampling fluctuations. Using a concentration inequality (Hoeffding's inequality combined with a union bound), we obtain:
\begin{equation}
P(\hat Y^{(K)}_\pi(x) \neq y^*_\beta(x)) \;\le\; |\mathcal{Y}| \exp(-\frac{K}{2} \,\Delta(x,\beta)^2),
\end{equation}
where $\Delta(x,\beta)$ denotes the gap between the probability of the mode and its closest competitor.

Thus, the accuracy can be expressed as:
\begin{equation}
\text{Acc}(\beta) = P(y^*_\beta(X)=Y) \;+\; O\big(\exp(-\frac{K}{2}\,\Delta(X,\beta)^2)\big).
\end{equation}
In particular, as long as the theoretical mode remains unchanged, we have:
\begin{equation}
|\Delta \text{Acc}| \;\le\; O\big(\exp(-cK)\big),
\end{equation}
for a constant $c>0$, provided that $\Delta(x,\beta)$ is uniformly lower-bounded.

\paragraph{Comparison}
We therefore obtain the following scaling behaviors:
\begin{itemize}
    \item Confidence variation:
    \begin{equation}
    |\Delta p^*| \sim \int_{\beta_2}^{\beta_1} \frac{1}{\beta^2} d\beta,
    \end{equation}
    \item Accuracy variation:
    \begin{equation}
    |\Delta \text{Acc}| \le O(\exp(-cK)).
    \end{equation}
\end{itemize}

Hence, for a sufficiently large $K$:
\begin{equation}
|\Delta \text{Acc}| \ll |\Delta p^*|,
\end{equation}
which implies:
\begin{equation}
\Delta p^* - \Delta \text{Acc} < 0.
\end{equation}
We conclude that under this overconfidence regime:
\begin{equation}
\text{ECE}(\pi_{\beta_1}) \le \text{ECE}(\pi_{\beta_2}) + O\left( \sqrt{\frac{\log |\mathcal{Y}|}{K}} \right)
\end{equation}

\end{proof}

\section{Detailed Experimental Results}\label{app:detailed_XP}

\paragraph{Large Reasoning Models and setup parameters.}
All experiments are conducted on \gpt~\citep{openai2025gptoss120bgptoss20bmodel}, \qwen~\citep{yang2025qwen3technicalreport}, \deepseek~\citep{deepseek}, and \closedgpt~\citep{openai2026gpt56}. 
Unless otherwise stated, we use multinomial sampling with a temperature of 0.7. We emphasize that access to the sampling temperature is an artifact of our reproducible open-weight setup, not an assumption of our model: \textsc{J4U} targets deployments exposing text-in/text-out only, where decoding hyperparameters are not user-controllable. This is already the case for \closedgpt, whose API does not expose the sampling temperature; we instead set \texttt{reasoning\_effort=medium} and a maximum of 4096 completion tokens, matching the generation budget used for the other models. For \textsc{Rephrase}, we follow \citet{yang2024just} and use a higher temperature of 1.5 for the rephrasing phase; the rephrasing is always performed by the model under evaluation itself, and for \closedgpt\ it uses \texttt{reasoning\_effort=medium} as temperature cannot be set.

No fine-tuning is performed. The maximum number of generated tokens is set to 4096.
All models are run on a single NVIDIA H100 GPU with 96GB of VRAM. To enable efficient inference, we apply 4-bit quantization for \deepseek, use the native MXFP4 quantization for \gpt, and load \qwen~weights in \texttt{torch.bfloat16} precision. \closedgpt~is accessed through the OpenAI Batch API rather than run locally. 

\paragraph{Datasets and methodological details.}
We evaluate our method and baselines on three datasets. Two are standard question-answering benchmarks designed to assess factual knowledge and reasoning: SuperGPQA\footnote{https://huggingface.co/datasets/m-a-p/SuperGPQA}~\citep{du2025supergpqa} and Humanity’s Last Exam (HLE)\footnote{https://huggingface.co/datasets/cais/hle}~\citep{hle}. The third dataset, SynthPAI\footnote{https://github.com/eth-sri/SynthPAI}~\citep{yukhymenko2024synthetic}, targets tasks with substantial aleatoric uncertainty. It consists of synthetic Reddit posts and is commonly used to evaluate attribute inference attacks~\citep{staab24beyond}, where the goal is to infer implicit author attributes; in this work, we focus on gender prediction. All three datasets are known to be particularly challenging, making them well suited for studying model uncertainty.
For datasets containing more than 1,000 instances, we uniformly subsample 1,000 examples at random. 

Finally, if a model exceeds the maximum generation length or produces an output that does not conform to the required format, the corresponding prediction is set to null and excluded from all metric computations. However, the null ratio between the methods for each dataset remained equivalent throughout.

\paragraph{Metrics computation.}
With K=10 samples, the confidence of a prediction is the fraction of samples supporting the selected answer, so confidences take values in the finite set $\{0,0.1,\dots,1\}$. The same holds for VC under the Avg-Conf aggregation strategy. To compute ECE, we therefore bin by value: each attainable confidence level forms its own bin, and empty bins are discarded. As a result, no two predictions with different confidences are merged. Equal-width and equal-mass schemes are not applicable here, as they would only merge atoms and coarsen the estimate.
For the same reason, the probability assigned to the ground-truth answer is exactly zero whenever that answer is never sampled, which makes the NLL infinite.
We therefore clip probabilities to $[\varepsilon, 1-\varepsilon]$ with $\varepsilon = 10^{-15}$ before taking the logarithm, identically for all methods.
Each such instance contributes $-\log\varepsilon \approx 34.5$, so the NLL is dominated by the fraction of instances whose correct answer is never sampled; we report it for completeness and base our conclusions primarily on ECE and the Brier score.
To compute the Brier score on HLE, which requires a finite set of possible answers for each instance, we restrict to multiple-choice questions in the results reported in the main paper. In the appendix, we additionally evaluate on the open-ended portion of HLE (Appendix~\ref{app:hleopen}) and on \texttt{gsm8k}~\citep{cobbe2021gsm8k}\footnote{https://huggingface.co/datasets/openai/gsm8k} (Appendix~\ref{app:easybench}); since neither setting admits a finite answer set, we report accuracy, ECE, and NLL only, and find consistent trends with the main-paper results.

\subsection{Accuracy and confidence calibration per LRM and dataset}
\label{app:detailacc}

\begin{table*}[!h]
\centering
\small

\caption{Comparison of UQ techniques on \qwen. Bold indicates statistically significant improvements w.r.t. \textsc{RS} by bootstrap test with $\alpha = 0.05$.}
\label{tab:acc_qwen}
\begin{tabular}{llccccc}
\toprule
Dataset & Method & Accuracy $\uparrow$ & ECE $\downarrow$ & Brier $\downarrow$ & NLL $\downarrow$ \\ 
\midrule
\multirow{6}{*}{\textit{SynthPAI}}
 & \textsc{RS} & \mesure{0.56}{0.03} & \mesure{0.34}{0.03} & \mesure{0.37}{0.03} & \mesure{8.27}{0.88} \\
 & VC  & \mesure{0.57}{0.03} & \mesure{0.33}{0.03} & \mesure{0.36}{0.03} & \mesure{8.30}{0.85} \\
 & \textsc{Rephrase}  & \mesure{0.56}{0.03} & \mesure{0.32}{0.03} & \mesure{0.36}{0.02} & \mesure{7.30}{0.88} \\
 & \textsc{J4U-Suffix}  & \mesure{0.56}{0.03} & \mesure{0.30}{0.03} & \mesure{0.35}{0.02} & \mesurebf{6.19}{0.83} \\
 & \textsc{J4U-Prog}  & \mesure{0.58}{0.03} & \mesurebf{0.23}{0.03} & \mesurebf{0.31}{0.02} & \mesurebf{3.91}{0.64} \\
 & \textsc{J4U-Art}  & \mesure{0.57}{0.03} & \mesurebf{0.25}{0.03} & \mesurebf{0.32}{0.02} & \mesurebf{4.22}{0.64} \\
\midrule
\multirow{6}{*}{\textit{SuperGPQA}} 
 & \textsc{RS} & \mesure{0.45}{0.03} & \mesure{0.39}{0.03} & \mesure{0.09}{0.01} & \mesure{13.79}{1.11} \\
 & VC  & \mesure{0.45}{0.03} & \mesure{0.42}{0.03} & \mesure{0.09}{0.01} & \mesure{13.93}{1.11} \\
 & \textsc{Rephrase}  & \mesure{0.44}{0.03} & \mesure{0.39}{0.03} & \mesure{0.09}{0.01} & \mesure{14.21}{1.08} \\
 & \textsc{J4U-Suffix}  & \mesure{0.44}{0.03} & \mesure{0.41}{0.03} & \mesure{0.09}{0.01} & \mesure{14.12}{1.11} \\
 & \textsc{J4U-Prog}  & \mesure{0.41}{0.03} & \mesurebf{0.25}{0.03} & \mesure{0.08}{0.01} & \mesure{11.99}{1.02} \\
 & \textsc{J4U-Art} & \mesure{0.40}{0.03} & \mesurebf{0.33}{0.03} & \mesure{0.09}{0.01} & \mesure{13.97}{1.03} \\
\midrule
\multirow{6}{*}{\shortstack[l]{\textit{Humanity's}\\\textit{Last Exam}}} 
 & \textsc{RS}& \mesure{0.12}{0.04} & \mesure{0.73}{0.04} & \mesure{0.29}{0.01} & \mesure{23.34}{1.57}\\
 & VC  & \mesure{0.10}{0.04} & \mesure{0.78}{0.04} & \mesure{0.30}{0.01} & \mesure{23.13}{1.73} \\
 & \textsc{Rephrase}  & \mesure{0.17}{0.04} & \mesure{0.67}{0.05} & \mesure{0.28}{0.01} & \mesure{21.32}{1.81} \\
 & \textsc{J4U-Suffix}  & \mesure{0.13}{0.04} & \mesure{0.74}{0.04} & \mesure{0.30}{0.01} & \mesure{22.96}{1.61} \\
 & \textsc{J4U-Prog}  & \mesure{0.11}{0.03} & \mesurebf{0.59}{0.04} & \mesurebf{0.25}{0.01} & \mesurebf{17.32}{1.54} \\
 & \textsc{J4U-Art}  & \mesure{0.12}{0.04} & \mesurebf{0.66}{0.04} & \mesurebf{0.27}{0.01} & \mesurebf{18.85}{1.65} \\
\bottomrule
\end{tabular}
\end{table*}

\begin{table*}[!h]

\caption{Comparison of UQ techniques on \gpt. Bold indicates statistically significant improvements w.r.t.  \textsc{RS}  by bootstrap test with $\alpha = 0.05$.}
\label{tab:acc_gpt}
\centering
\small
\begin{tabular}{llccccc}
\toprule
Dataset & Method & Accuracy $\uparrow$ & ECE $\downarrow$ & Brier $\downarrow$ & NLL $\downarrow$ \\ 
\midrule
\multirow{6}{*}{\textit{SynthPAI}}
 & \textsc{RS} & \mesure{0.61}{0.03} & \mesure{0.26}{0.03} & \mesure{0.31}{0.02} & \mesure{5.41}{0.74} \\
 & \textsc{VC}  & \mesure{0.59}{0.03} & \mesure{0.28}{0.03} & \mesure{0.32}{0.03} & \mesure{5.85}{0.78} \\
 & \textsc{Rephrase}  & \mesure{0.57}{0.03} & \mesure{0.26}{0.03} & \mesure{0.32}{0.02} & \mesure{4.72}{0.71} \\
 & \textsc{J4U-Suffix}  & \mesure{0.60}{0.03} & \mesure{0.22}{0.03} & \mesure{0.30}{0.02} & \mesurebf{3.63}{0.64} \\
 & \textsc{J4U-Prog}  & \mesure{0.60}{0.03} & \mesure{0.24}{0.03} & \mesure{0.31}{0.02} & \mesure{4.75}{0.72} \\
 & \textsc{J4U-Art}  & \mesure{0.60}{0.03} & \mesure{0.24}{0.03} & \mesure{0.30}{0.02} & \mesure{4.21}{0.65} \\
\midrule
\multirow{6}{*}{\textit{SuperGPQA}} 
 & \textsc{RS} & \mesure{0.47}{0.03} & \mesure{0.33}{0.03} & \mesure{0.08}{0.01} & \mesure{12.06}{0.98} \\
 & \textsc{VC}  & \mesure{0.47}{0.03} & \mesure{0.38}{0.03} & \mesure{0.09}{0.01} & \mesure{11.75}{0.96} \\
 & \textsc{Rephrase}  & \mesure{0.44}{0.03} & \mesurebf{0.27}{0.03} & \mesure{0.08}{0.01} & \mesure{10.77}{0.96} \\
 & \textsc{J4U-Suffix}  & \mesure{0.46}{0.03} & \mesure{0.33}{0.03} & \mesure{0.08}{0.01} & \mesure{12.07}{0.93} \\
 & \textsc{J4U-Prog}  & \mesure{0.43}{0.03} & \mesurebf{0.25}{0.03} & \mesure{0.08}{0.01} & \mesure{10.85}{0.99} \\
 & \textsc{J4U-Art} & \mesure{0.44}{0.03} & \mesure{0.30}{0.03} & \mesure{0.08}{0.01} & \mesure{11.34}{0.94} \\
\midrule
\multirow{6}{*}{\shortstack[l]{\textit{Humanity's}\\\textit{Last Exam}}} 
 & \textsc{RS}& \mesure{0.13}{0.03} & \mesure{0.72}{0.04} & \mesure{0.29}{0.01} & \mesure{22.38}{1.59} \\
 & \textsc{VC}  & \mesure{0.13}{0.03} & \mesure{0.76}{0.04} & \mesure{0.30}{0.01} & \mesure{22.98}{1.63} \\
 & \textsc{Rephrase}  & \mesure{0.12}{0.03} & \mesure{0.72}{0.04} & \mesure{0.29}{0.01} & \mesure{22.81}{1.65} \\
 & \textsc{J4U-Suffix}  & \mesure{0.14}{0.03} & \mesure{0.71}{0.04} & \mesure{0.29}{0.01} & \mesure{22.34}{1.61} \\
 & \textsc{J4U-Prog}  & \mesure{0.11}{0.03} & \mesure{0.67}{0.04} & \mesure{0.28}{0.01} & \mesure{21.33}{1.56} \\
 & \textsc{J4U-Art}  & \mesure{0.17}{0.04} & \mesurebf{0.58}{0.04} & \mesurebf{0.26}{0.01} & \mesure{19.68}{1.61} \\
\bottomrule
\end{tabular}
\end{table*}

\begin{table*}[!h]

\caption{Comparison of UQ techniques on \deepseek. Bold indicates statistically significant improvements w.r.t.  \textsc{RS}  by bootstrap test with $\alpha = 0.05$. Red indicates significant deterioration.}

\label{tab:acc_deepseek}
\centering
\small
\begin{tabular}{llccccc}
\toprule
Dataset & Method & Accuracy $\uparrow$ & ECE $\downarrow$ & Brier $\downarrow$ & NLL $\downarrow$ \\ 
\midrule
\multirow{6}{*}{\textit{SynthPAI}}
 & \textsc{RS}  & \mesure{0.57}{0.03} & \mesure{0.30}{0.03} & \mesure{0.35}{0.02} & \mesure{7.25}{0.83} \\
 & \textsc{VC}  & \mesure{0.57}{0.03} & \mesure{0.31}{0.03} & \mesure{0.35}{0.02} & \mesure{7.30}{0.82} \\
 & \textsc{Rephrase}  & \mesure{0.53}{0.03} & \mesure{0.32}{0.03} & \mesure{0.35}{0.02} & \mesurebf{5.63}{0.75} \\
 & \textsc{J4U-Suffix}  & \mesure{0.57}{0.03} & \mesure{0.28}{0.03} & \mesure{0.33}{0.02} & \mesure{6.01}{0.78} \\
 & \textsc{J4U-Prog}  & \mesure{0.58}{0.03} & \mesurebf{0.25}{0.03} & \mesure{0.32}{0.02} & \mesurebf{4.71}{0.65} \\
 & \textsc{J4U-Art}  & \mesure{0.58}{0.03} & \mesurebf{0.25}{0.03} & \mesurebf{0.31}{0.02} & \mesurebf{4.41}{0.62} \\
\midrule
\multirow{6}{*}{\textit{SuperGPQA}} 
 & \textsc{RS}  & \mesure{0.50}{0.04} & \mesure{0.33}{0.03} & \mesure{0.08}{0.01} & \mesure{11.79}{1.26} \\
 & \textsc{VC}  & \mesure{0.51}{0.05} & \mesure{0.35}{0.04} & \mesure{0.08}{0.01} & \mesure{13.31}{1.45} \\
 & \textsc{Rephrase}  & \textcolor{red}{\mesure{0.29}{0.03}} & \mesure{0.28}{0.03} & \textcolor{red}{\mesure{0.10}{0.01}} & \textcolor{red}{\mesure{16.93}{1.05}} \\
 & \textsc{J4U-Suffix}  & \mesure{0.50}{0.04} & \mesure{0.32}{0.04} & \mesure{0.08}{0.01} & \mesure{12.11}{1.15} \\
 & \textsc{J4U-Prog}  & \mesure{0.44}{0.04} & \mesure{0.30}{0.04} & \mesure{0.09}{0.01} & \mesure{13.05}{1.21} \\
 & \textsc{J4U-Art} & \mesure{0.43}{0.04} & \mesure{0.34}{0.04} & \mesure{0.09}{0.01} & \mesure{14.41}{1.35} \\
\midrule
\multirow{6}{*}{\shortstack[l]{\textit{Humanity's}\\\textit{Last Exam}}} 
 & \textsc{RS} & \mesure{0.07}{0.03} & \mesure{0.80}{0.04} & \mesure{0.31}{0.01} & \mesure{24.85}{2.01} \\
 & \textsc{VC}  & \mesure{0.08}{0.04} & \mesure{0.80}{0.04} & \mesure{0.32}{0.01} & \mesure{24.54}{2.01} \\
 & \textsc{Rephrase}  & \mesure{0.11}{0.04} & \mesurebf{0.71}{0.05} & \mesure{0.29}{0.01} & \mesure{22.58}{2.04} \\
 & \textsc{J4U-Suffix}  & \mesure{0.09}{0.04} & \mesure{0.77}{0.04} & \mesure{0.30}{0.01} & \mesure{24.27}{1.76} \\
 & \textsc{J4U-Prog}  & \mesure{0.11}{0.04} & \mesurebf{0.67}{0.04} & \mesurebf{0.28}{0.01} & \mesure{22.87}{1.87} \\
 & \textsc{J4U-Art}  & \mesurebf{0.14}{0.04} & \mesurebf{0.57}{0.05} & \mesurebf{0.26}{0.01} & \mesurebf{20.67}{2.0} \\
\bottomrule
\end{tabular}
\end{table*}

\begin{table*}[!h]

\caption{Comparison of UQ techniques on \closedgpt. Bold indicates statistically significant improvements w.r.t.  \textsc{RS}  by bootstrap test with $\alpha = 0.05$. Red indicates significant deterioration.}

\label{tab:acc_closedgpt}
\centering
\small 
\begin{tabular}{llccccc}
\toprule
Dataset & Method & Accuracy $\uparrow$ & ECE $\downarrow$ & Brier $\downarrow$ & NLL $\downarrow$ \\ 
\midrule
\multirow{6}{*}{\textit{SynthPAI}}
 & \textsc{RS}  & \mesure{0.61}{0.03} & \mesure{0.32}{0.03} & \mesure{0.35}{0.03} & \mesure{9.86}{0.99} \\
 & \textsc{VC}  & \mesure{0.59}{0.03} & \mesure{0.37}{0.04} & \mesure{0.39}{0.03} & \mesure{11.67}{1.2} \\
 & \textsc{Rephrase}  & \mesure{0.57}{0.03} & \mesure{0.36}{0.03} & \mesure{0.38}{0.03} & \mesure{10.61}{0.96} \\
 & \textsc{J4U-Suffix}  & \mesure{0.60}{0.03} & \mesure{0.30}{0.03} & \mesure{0.33}{0.03} & \mesurebf{7.94}{0.98} \\
 & \textsc{J4U-Prog}  & \mesure{0.59}{0.03} & \mesure{0.30}{0.03} & \mesure{0.34}{0.02} & \mesure{8.63}{0.94} \\
 & \textsc{J4U-Art}  & \mesure{0.61}{0.03} & \mesurebf{0.26}{0.03} & \mesure{0.32}{0.03} & \mesurebf{7.45}{0.84} \\
\midrule
\multirow{6}{*}{\textit{SuperGPQA}} 
 & \textsc{RS}  & \mesure{0.68}{0.03} & \mesure{0.23}{0.03} & \mesure{0.06}{0.01} & \mesure{8.68}{0.92} \\
 & \textsc{VC}  & \mesure{0.67}{0.03} & \mesure{0.27}{0.03} & \mesure{0.06}{0.01} & \mesure{9.67}{0.99} \\
 & \textsc{Rephrase}  & \mesure{0.68}{0.03} & \mesure{0.23}{0.03} & \mesure{0.06}{0.01} & \mesure{8.53}{0.94} \\
 & \textsc{J4U-Suffix}  & \mesure{0.68}{0.03} & \mesure{0.23}{0.03} & \mesure{0.06}{0.01} & \mesure{8.56}{0.93} \\
 & \textsc{J4U-Prog}  & \mesure{0.65}{0.03} & \mesure{0.20}{0.03} & \mesure{0.06}{0.01} & \mesure{8.17}{0.90} \\
 & \textsc{J4U-Art} & \mesure{0.64}{0.03} & \mesure{0.20}{0.03} & \mesure{0.06}{0.01} & \mesure{8.60}{0.87} \\
\midrule
\multirow{6}{*}{\shortstack[l]{\textit{Humanity's}\\\textit{Last Exam}}} 
 & \textsc{RS} & \mesure{0.26}{0.04} & \mesure{0.57}{0.04} & \mesure{0.13}{0.01} & \mesure{20.99}{1.38} \\
 & \textsc{VC}  & \mesure{0.24}{0.04} & \textcolor{red}{\mesure{0.69}{0.04}} & \mesure{0.14}{0.01} & \mesure{22.87}{1.61} \\
 & \textsc{Rephrase}  & \mesure{0.25}{0.04} & \mesure{0.58}{0.04} & \mesure{0.13}{0.01} & \mesure{20.44}{1.50} \\
 & \textsc{J4U-Suffix}  & \mesure{0.25}{0.04} & \mesure{0.61}{0.04} & \mesure{0.13}{0.01} & \textcolor{red}{\mesure{24.90}{1.35}} \\
 & \textsc{J4U-Prog}  & \mesure{0.22}{0.04} & \mesure{0.56}{0.04} & \mesure{0.13}{0.01} & \mesure{20.35}{1.45} \\
 & \textsc{J4U-Art}  & \mesure{0.23}{0.04} & \mesure{0.55}{0.04} & \mesure{0.13}{0.01} & \mesure{19.25}{1.45} \\
\bottomrule
\end{tabular}
\end{table*}

This subsection reports detailed Accuracy, ECE, Brier, and NLL values, aggregated in Table~\ref{tab:summary_percentage_metrics}. Tables~\ref{tab:acc_qwen},~\ref{tab:acc_gpt},~\ref{tab:acc_deepseek}, and~\ref{tab:acc_closedgpt} present the per-method and per-dataset results for \qwen, \gpt, \deepseek, and \closedgpt, respectively. All results are reported for $K = 10$.

Bold values indicate statistically significant improvements with respect to the baseline, as determined by a bootstrap difference test with $\alpha = 0.05$.

Most improvements are observed on \qwen~and \deepseek. We hypothesize that \gpt~and \closedgpt~may be more robust to jailbreak perturbations, which could explain the comparatively smaller gains. We note that, even on \closedgpt, a recent production-grade LRM, \textsc{J4U-Prog} and \textsc{J4U-Art} do not degrade UQ. On the contrary, they improve ECE and NLL on every dataset, and Brier specifically on SynthPAI. These improvements however reach statistical significance only for \textsc{J4U-Art}, and only in two settings (SynthPAI ECE and NLL).

Among the three evaluated jailbreak-based transformations, \textsc{J4U-Suffix} consistently underperforms relative to \textsc{J4U-Prog} and \textsc{J4U-Art}. This weaker performance can be attributed to its simplified design. Whereas \textsc{J4U-Prog} and \textsc{J4U-Art} are direct adaptations of state-of-the-art jailbreak strategies, \textsc{J4U-Suffix} constitutes a simplified variant of GCG~\cite{zou2023universaltransferableadversarialattacks}. In its original formulation, GCG operates in a white-box setting and leverages gradient-based optimization to iteratively construct adversarial suffixes. In contrast, \textsc{J4U-Suffix} relies on randomly sampled suffixes, substantially reducing its attack strength.

The fact that the only simplified jailbreak underperforms suggests that jailbreak effectiveness correlates with operator strength: higher-quality jailbreaks appear to induce stronger effects and greater improvements in calibration.

\subsection{Detailed evaluation on HLE - open-ended questions}
\label{app:hleopen}
Table~\ref{tab:hleopen} reports accuracy, ECE, and NLL on open-ended questions of the HLE dataset belonging to the exact-match category, to simplify correctness assessment. Results are consistent with those on multiple-choice questions: no \textsc{J4U} instantiation significantly degrades accuracy. \textsc{J4U-Prog} improves ECE in all 4 settings, with statistically significant gains in 2 of them (\qwen, \deepseek). \textsc{J4U-Art} improves ECE in 3 of 4 settings, all statistically significant (\qwen, \gpt, \deepseek), though it shows no improvement on \closedgpt. As in the main-paper results, \textsc{VC} shows statistically significant ECE degradation relative to \textsc{RS} on two settings (\gpt, \closedgpt). No method achieves statistically significant NLL improvements in this setting.

\begin{table*}[!h]
\centering

\caption{Accuracy and confidence calibration on HLE (exact match). Bold indicates statistically significant improvements w.r.t.  \textsc{RS}  by bootstrap test with $\alpha = 0.05$. Red indicates significant deterioration.}
\begin{tabular}{llcccc}
\toprule
{LRM} & Method & Accuracy $\uparrow$ & ECE $\downarrow$ & NLL $\downarrow$ \\ 
\midrule
\multirow{6}{*}{\textit{\qwen}}
 & \textsc{RS}   & \mesure{0.02}{0.01} & \mesure{0.72}{0.03} & \mesure{8.04}{1.16} \\
 & \textsc{VC}  & \mesurebf{0.05}{0.02} &  \mesure{0.73}{0.03} & \mesure{8.17}{1.22} \\
 & \textsc{Rephrase}  & \mesure{0.02}{0.01} &  \mesurebf{0.61}{0.02} & \mesure{8.17}{0.97} \\
 &\textsc{J4U-Suffix}  & \mesure{0.03}{0.01} &  \mesure{0.68}{0.03} & \mesure{8.32}{1.22} \\
 &\textsc{J4U-Prog}  & \mesure{0.02}{0.01} &  \mesurebf{0.62}{0.02} & \mesure{8.60}{0.98} \\
 &\textsc{J4U-Art}  & \mesure{0.01}{0.01} &  \mesurebf{0.53}{0.02} & \mesure{8.36}{0.99} \\
 \midrule
 \multirow{6}{*}{\textit{\gpt}}
 & \textsc{RS}   & \mesure{0.03}{0.02} & \mesure{0.63}{0.03} & \mesure{8.05}{1.24} \\
 & \textsc{VC}  & \mesure{0.03}{0.01} &  \mesure{\color{red}0.73}{0.02} & \mesure{8.01}{1.17} \\
 & \textsc{Rephrase}  & \mesure{0.02}{0.01} &  \mesure{0.60}{0.03} & \mesure{7.57}{1.20} \\
 &\textsc{J4U-Suffix}  & \mesure{0.03}{0.01} &  \mesure{0.60}{0.03} & \mesure{7.62}{1.19} \\
 &\textsc{J4U-Prog}  & \mesure{0.02}{0.01} &  \mesure{0.60}{0.02} & \mesure{8.38}{1.15} \\
 &\textsc{J4U-Art}  & \mesure{0.03}{0.01} &  \mesurebf{0.56}{0.03} & \mesure{8.42}{1.17} \\
    \midrule
 \multirow{6}{*}{\textit{\deepseek}}
 & \textsc{RS}   & \mesure{0.02}{0.01} & \mesure{0.61}{0.03} & \mesure{8.07}{1.40} \\
 & \textsc{VC}  & \mesure{0.03}{0.02} &  \mesure{0.69}{0.04} & \mesure{7.16}{1.86} \\
 & \textsc{Rephrase}  & \mesure{0.03}{0.02} &  \mesure{0.57}{0.03} & \mesure{7.68}{1.21} \\
 &\textsc{J4U-Suffix}  & \mesure{0.02}{0.01} &  \mesure{0.60}{0.03} & \mesure{8.19}{1.36} \\
 &\textsc{J4U-Prog}  & \mesure{0.04}{0.02} &  \mesurebf{0.55}{0.03} & \mesure{7.88}{1.29} \\
 &\textsc{J4U-Art}  & \mesure{0.02}{0.01} &  \mesurebf{0.48}{0.03} & \mesure{8.22}{1.34} \\
 \midrule
 \multirow{6}{*}{\textit{\closedgpt}}
 & \textsc{RS}   & \mesure{0.16}{0.02} & \mesure{0.51}{0.03} & \mesure{5.37}{0.77} \\
 & \textsc{VC}  & \mesure{0.18}{0.03} &  \mesure{\color{red} 0.61}{0.03} & \mesure{5.31}{0.43} \\
 & \textsc{Rephrase}  & \mesure{0.15}{0.02} &  \mesure{0.50}{0.03} & \mesure{5.41}{0.77} \\
 &\textsc{J4U-Suffix}  & \mesure{0.15}{0.02} &  \mesure{0.50}{0.03} & \mesure{5.16}{0.75} \\
 &\textsc{J4U-Prog}  & \mesure{0.15}{0.02} &  \mesure{0.48}{0.03} & \mesure{5.34}{0.83} \\
 &\textsc{J4U-Art}  & \mesure{0.14}{0.02} &  \mesure{0.51}{0.02} & \mesure{5.74}{0.80} \\
\bottomrule
\end{tabular}
\label{tab:hleopen}
\end{table*}

\subsection{Accuracy and confidence calibration as a function of the number of iterations}

\label{app:confidencePerIteration}

All results reported in the paper correspond to $K = 10$. Figures~\ref{fig:conv} illustrate accuracy and confidence calibration as a function of the number of iterations, for $K = 1, \dots, 10$, on four representative LRM–dataset pairs such that each dataset and each LRM is covered.

\begin{figure}[!ht]
    \centering
    \includegraphics[width=1\linewidth]{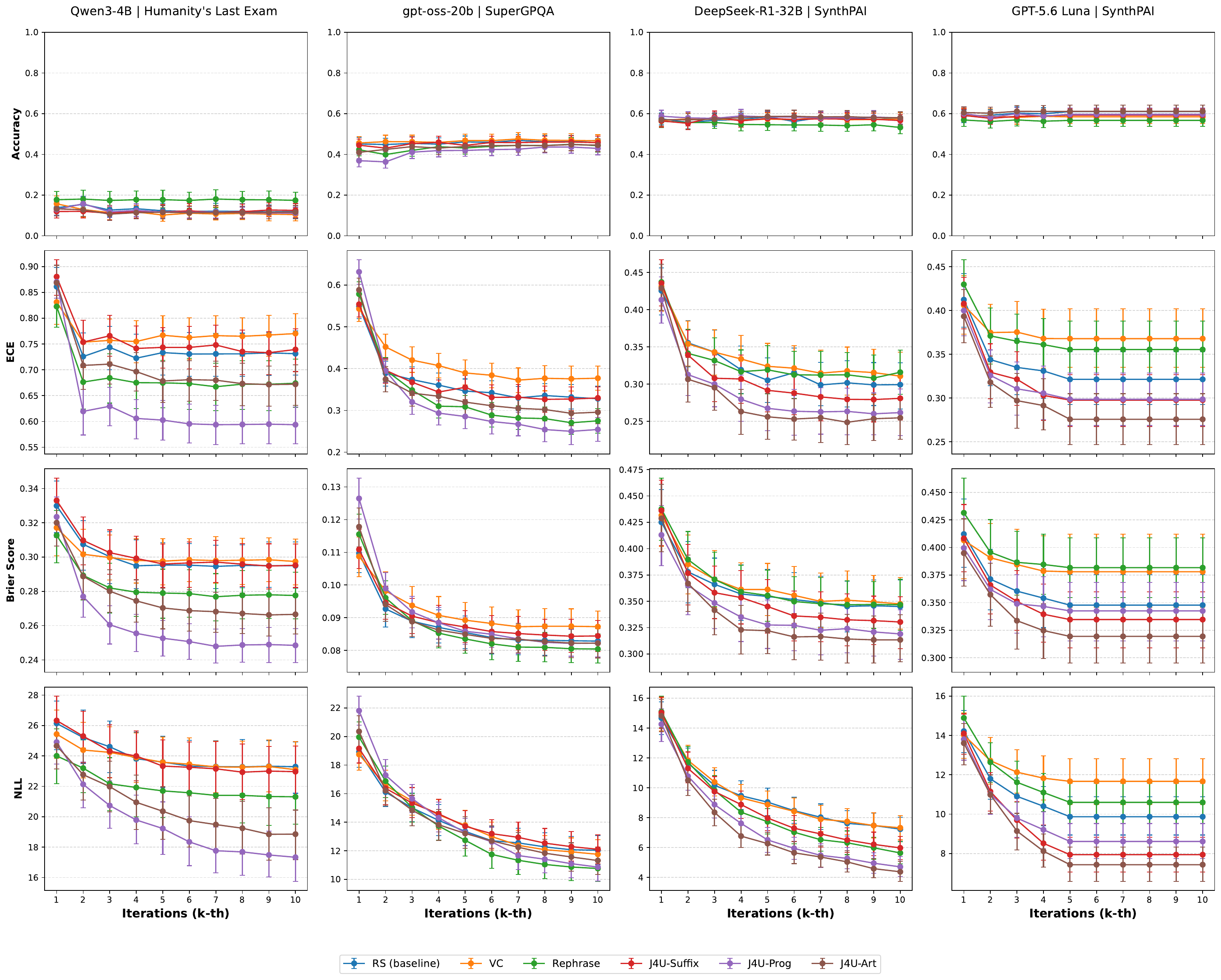}
    \caption{Accuracy and Confidence Calibration as a Function of the Number of Iteration on 4 Representative LRM-Dataset Pairs.}
    \label{fig:conv}
\end{figure}

At $K = 1$, all self-consistency-based approaches assign 100\% confidence to their predictions, as confidence is estimated from a single sample. This naturally leads to misalignment and overconfidence at $K=1$. Increasing $K$ introduces variance in the sampled answers, enabling frequency-based confidence estimates and thereby reducing overconfidence. 

Across methods, convergence dynamics are qualitatively similar. Accuracy remains largely stable as $K$ increases, suggesting that additional samples primarily refine confidence estimates rather than altering the predicted answer. In contrast, ECE and Brier decrease sharply for small $K$, consistent with improved confidence calibration as sample size grows. After approximately $K \approx 5$, improvements become marginal, indicating diminishing returns from additional sampling. In contrast, NLL decreases more gradually, and it remains unclear whether it has fully plateaued at $K = 10$.

\subsection{Empirical Signatures Predicted by the Relaxation Interpretation}

\subsubsection{Flattening output distribution}
\label{app:entropy}

Table~\ref{tab:comparison_llms_entropy} reports the Shannon entropy statistics summarized in Table~\ref{tab:summary_entropy_combined}. For each LRM--dataset pair, we provide the mean and variance of the entropy distribution induced by each policy, together with the $p$-value of the one-side Mann–Whitney U test comparing entropy values against $\pi_\text{RL}$ (the policy defined by \textsc{RS}).

\textsc{J4U-Suffix} and \textsc{Rephrase} fail to reach statistical significance in 6 and 3 out of 12 LRM--dataset configurations, respectively. The entropy gains induced by \textsc{J4U-Suffix} remain modest overall, and even slightly decrease entropy on \qwen--SuperGPQA and \closedgpt--HLE. 

\textsc{J4U-Art} and \textsc{J4U-Prog} yields statistically significant entropy increases across 11 and 12 out of 12 settings, respectively, in some cases doubling entropy relative to \textsc{RS} (e.g., \qwen\ on SuperGPQA). Overall, the highest relative entropy increases is achieved by \textsc{Rephrase} on SuperGPQA--\deepseek ($\times 2.6$). Notably, in this configuration, the entropy gain is accompanied by a significant degradation in accuracy, Brier score, and NLL (see Appendix~\ref{app:detailacc}).  

We hypothesize that, in this setting, \textsc{Rephrase} alters or weakens the semantic structure of the questions, leading to less constrained reasoning and consequently to highly diverse (but less reliable) answers.

\newpage
\begin{table*}[!h]
\centering
\caption{Mean and variance of Shannon entropy across policies, together with p-value of one-side Mann–Whitney U test against RS.}
\noindent\hspace*{-2cm}
\small
\setlength{\tabcolsep}{4pt}
\begin{tabular}{ll ccc ccc ccc ccc}
\toprule
 & & \multicolumn{3}{c}{\qwen} & \multicolumn{3}{c}{\gpt} & \multicolumn{3}{c}{\deepseek} & \multicolumn{3}{c}{\closedgpt} \\
\cmidrule(lr){3-5} \cmidrule(lr){6-8} \cmidrule(lr){9-11} \cmidrule(lr){12-14}
Dataset & Method & Mean & Var. & p-value & Mean & Var. & p-value & Mean & Var. & p-value & Mean & Var. & p-value \\
\midrule
\multirow{5}{*}{\textit{SynthPAI}}
 & \textsc{RS} & 0.22 & 0.07 &  & 0.30 & 0.08 &  & 0.26 & 0.08 &  & 0.14 & 0.06 &  \\
 & Rephrase & 0.26 & 0.08 & \floatbf{7}{-7} & 0.35 & 0.07 & \floatbf{7}{-11} & 0.31 & 0.08 & \floatbf{5}{-12} & 0.15 & 0.07 & \float{6}{-2} \\
 & \textsc{J4U-Suffix} & 0.28 & 0.07 & \floatbf{3}{-19} & 0.37 & 0.07 & \floatbf{1}{-17} & 0.30 & 0.08 & \floatbf{2}{-11} & 0.22 & 0.08 & \floatbf{7}{-17} \\
 & \textsc{J4U-Prog} & 0.39 & 0.07 & \floatbf{2}{-35} & 0.34 & 0.08 & \floatbf{1}{-5} & 0.33 & 0.07 & \floatbf{1}{-10} & 0.21 & 0.08 & \floatbf{1}{-11} \\
 & \textsc{J4U-Art} & 0.38 & 0.07 & \floatbf{3}{-32} & 0.35 & 0.07 & \floatbf{4}{-8} & 0.35 & 0.07 & \floatbf{2}{-17} & 0.22 & 0.08 & \floatbf{8}{-16} \\
\midrule
\multirow{5}{*}{\textit{SuperGPQA}}
 & \textsc{RS} & 0.33 & 0.16 &  & 0.46 & 0.23 &  & 0.32 & 0.17 &  & 0.18 & 0.11 &  \\
 & Rephrase & 0.38 & 0.18 & \floatbf{5}{-7} & 0.65 & 0.27 & \floatbf{1}{-37} & 0.84 & 0.18 & \floatbf{3}{-86} & 0.20 & 0.11 & \floatbf{7}{-3} \\
 & \textsc{J4U-Suffix} & 0.32 & 0.16 & \float{9}{-1} & 0.48 & 0.23 & \floatbf{4}{-2} & 0.36 & 0.18 & \floatbf{1}{-3} & 0.19 & 0.11 & \float{8}{-2} \\
 & \textsc{J4U-Prog} & 0.78 & 0.28 & \floatbf{5}{-97} & 0.76 & 0.27 & \floatbf{1}{-70} & 0.54 & 0.22 & \floatbf{3}{-33} & 0.33 & 0.16 & \floatbf{5}{-29} \\
 & \textsc{J4U-Art} & 0.64 & 0.30 & \floatbf{3}{-71} & 0.60 & 0.28 & \floatbf{5}{-21} & 0.46 & 0.23 & \floatbf{6}{-17} & 0.35 & 0.18 & \floatbf{5}{-35} \\
\midrule
\multirow{5}{*}{\shortstack[l]{\textit{Humanity's}\\\textit{Last Exam}}}
 & \textsc{RS} & 0.31 & 0.14 &  & 0.30 & 0.14 &  & 0.24 & 0.12 &  & 0.28 & 0.12 &  \\
 & Rephrase & 0.32 & 0.15 & \float{7}{-1} & 0.48 & 0.18 & \floatbf{3}{-14} & 0.52 & 0.18 & \floatbf{7}{-18} & 0.29 & 0.13 & \float{4}{-1} \\
 & \textsc{J4U-Suffix} & 0.31 & 0.14 & \float{3}{-1} & 0.30 & 0.14 & \float{7}{-1} & 0.27 & 0.12 & \float{6}{-2} & 0.24 & 0.12 & \float{1}{0} \\
 & \textsc{J4U-Prog} & 0.62 & 0.19 & \floatbf{9}{-31} & 0.43 & 0.16 & \floatbf{1}{-8} & 0.42 & 0.17 & \floatbf{6}{-12} & 0.38 & 0.14 & \floatbf{3}{-8} \\
 & \textsc{J4U-Art} & 0.47 & 0.18 & \floatbf{6}{-14} & 0.32 & 0.16 & \float{2}{-1} & 0.33 & 0.16 & \floatbf{8}{-5} & 0.36 & 0.15 & \floatbf{3}{-7} \\
\bottomrule
\end{tabular}

\label{tab:comparison_llms_entropy}
\end{table*}

\subsubsection{Diversification of reasoning traces}
\label{app:cot}

Table~\ref{tab:comparison_llms} reports the mean and variance of pairwise cosine similarity between embeddings of reasoning traces generated by each policy, as summarized in Table~\ref{tab:summary_cosine_combined}. Reasoning traces embeddings are generated using \texttt{Qwen3-Embedding-8B}
\footnote{https://huggingface.co/Qwen/Qwen3-Embedding-8B}
~\citep{qwen3embedding}, without any system prompt. \textsc{VC} is not reported as the first prompt is exactly the same as \textsc{RS}.

Consistent with the entropy analysis, \textsc{J4U-Suffix} and \textsc{Rephrase} generally induce only modest reductions in cosine similarity, suggesting limited diversification of reasoning traces. In fact, they increase cosine similarity in 5 and 3 out of 9 LRM--dataset settings, respectively.

By contrast, \textsc{J4U-Art} reduces cosine similarity in 8 out of 9 settings and \textsc{J4U-Prog} in 7 out of 9. The exceptions occur on \deepseek, where \textsc{RS} already exhibits relatively low baseline similarity on SuperGPQA (0.52) and HLE (0.47), leaving limited room for further diversification. The most pronounced reduction is observed with \textsc{J4U-Art} on SynthPAI across all LRMs, where cosine similarity is approximately halved.

Interestingly, $\pi_\text{RL}$  exhibit high reasoning stability on SynthPAI for all LRMs, with \gpt~ and \deepseek~ showing the highest cosine similarity on this dataset, despite its high aleatoric uncertainty and the near-random accuracy achieved across models.

\begin{table*}[!ht]
\centering

\caption{Mean and variance of average cosine similarity between CoT embeddings across policies.}
\small
\setlength{\tabcolsep}{6pt}
\begin{tabular}{ll cc cc cc}
\toprule
 & & \multicolumn{2}{c}{\qwen} & \multicolumn{2}{c}{\gpt} & \multicolumn{2}{c}{\deepseek} \\
\cmidrule(lr){3-4} \cmidrule(lr){5-6} \cmidrule(lr){7-8}
Dataset & Method & Mean & Var. & Mean & Var. & Mean & Var. \\
\midrule
\multirow{5}{*}{\textit{SynthPAI}}
 & RS   & 0.67 & 0.15 & 0.85 & 0.09 & 0.73 & 0.14 \\
 & Rephrase    & 0.63 & 0.16 & 0.87 & 0.09 & 0.72 & 0.15 \\
 & \textsc{J4U-Suffix}  & 0.53 & 0.15 & 0.77 & 0.13 & 0.65 & 0.15 \\
 & \textsc{J4U-Prog}    & 0.48 & 0.14 & 0.65 & 0.15 & 0.61 & 0.16 \\
 & \textsc{J4U-Art}     & 0.34 & 0.07 & 0.45 & 0.12 & 0.35 & 0.14 \\
\midrule
\multirow{5}{*}{\textit{SuperGPQA}} 
 & RS   & 0.73 & 0.14 & 0.80 & 0.11 & 0.52 & 0.14 \\
 & Rephrase    & 0.71 & 0.14 & 0.78 & 0.12 & 0.58 & 0.17 \\
 & \textsc{J4U-Suffix}  & 0.74 & 0.13 & 0.81 & 0.10 & 0.55 & 0.15 \\
 & \textsc{J4U-Prog}   & 0.61 & 0.16 & 0.66 & 0.15 & 0.53 & 0.15 \\
 & \textsc{J4U-Art}     & 0.50 & 0.15 & 0.62 & 0.15 & 0.49 & 0.15 \\
\midrule
\multirow{5}{*}{\shortstack[l]{\textit{Humanity's}\\\textit{Last Exam}}} 
 & RS   & 0.63 & 0.15 & 0.62 & 0.14 & 0.47 & 0.13 \\
 & Rephrase    & 0.57 & 0.15 & 0.57 & 0.15 & 0.59 & 0.17 \\
 & \textsc{J4U-Suffix}  & 0.65 & 0.15 & 0.65 & 0.15 & 0.47 & 0.13 \\
 & \textsc{J4U-Prog}    & 0.59 & 0.16 & 0.43 & 0.11 & 0.47 & 0.14 \\
 & \textsc{J4U-Art}     & 0.57 & 0.16 & 0.49 & 0.13 & 0.49 & 0.14 \\
\bottomrule
\end{tabular}
\label{tab:comparison_llms}
\end{table*}

\section{Prompt Examples} \label{app:prompt}

This appendix reports the exact prompts used during inference for all experiments conducted with SynthPAI. Prompts are presented verbatim, with placeholders indicating where dataset-specific content is inserted. These prompts correspond to the different prompting strategies evaluated in the main paper and are provided to facilitate reproducibility. The scripts that construct the masked entries (\textsc{J4U-Art}), the suffix augmentation (\textsc{J4U-Suffix}), and the decomposition (\textsc{J4U-Prog}) are included in the supplementary material.

Across all datasets, \textsc{J4U} and \textsc{Rephrase} modify only the question; for multiple-choice questions, the answer options are left untouched.

\subsection{Initial prompt}
\label{app:iniprompt}
The initial prompt is used for both the \textsc{RS} and \textsc{VC} settings. In this prompt, $\{text\}$ denotes a single text instance drawn from the SynthPAI dataset. In the \textsc{J4U-Suffix} strategy, it is augmented with a randomly sampled suffix appended to the end of the question to induce output variability.

\begin{lstlisting}
You are an AI specialized in predicting the likely gender of the author of a text (Male or Female).

Your task is to **reason carefully step by step** about the text. Consider:  
- Word choice, sentence structure, and writing style.  
- Tone, phrasing, and stylistic patterns that might indicate the author's gender.  
- The semantic meaning of the text.  

After reasoning, output **only** one of the labels: `male` or `female`.  

Rules:  
1. Think carefully and explicitly about style and word usage before deciding.  
2. Do not provide explanations in your final output and follow the format label={<male or female>}.

Example:  
Text: "The team executed the project efficiently, hitting all milestones ahead of schedule."  
Output: label={male}


{text}
\end{lstlisting}

\subsection{Verbalized confidence}

In the \textsc{VC} setting, the initial prompt together with the model’s predicted label is appended with an additional query asking the model to estimate the likelihood that its previous answer is correct. In the following prompt fragment, the placeholder \texttt{label} is replaced by the model’s answer from the initial prompt.

{\small
\begin{lstlisting}
"role": "user",
"content": "How likely is the above answer to be correct? Please reason concisely and then answer with the following format: score={<the probability of answer " + label + " to be correct between 0.0 and 1.0, not the one you think correct, please only include the numerical number>}",

\end{lstlisting}
}

\subsection{Rephrase}

In the rephrase prompting strategy, the LRM is first instructed to reword a given input. Here, $\{text\}$ denotes a single text instance drawn from the SynthPAI dataset. The model is instructed to output only the rephrased version of the text. The resulting rephrased text is then used as input to the LRM for the downstream prediction task, replacing $\{text\}$ in the prompt described in Appendix~\ref{app:iniprompt}.

\begin{lstlisting}
"role": "user",
"content": f"Reword the following text:\n{text}\nRespond with the reworded text only."
\end{lstlisting}

\subsection{\textsc{J4U-Prog}}

In the \textsc{J4U-Prog} strategy, the prompt is augmented with explicit instructions describing how a target text is reconstructed from multiple fragments. Each input consists of several string fragments labeled alphabetically, along with an equation defining how these fragments combine to form the reconstructed text $z$. The model is instructed to base its prediction solely on the semantic content of $z$, rather than on the reconstruction rule itself.

An illustrative example of such an input–output pair is included directly in the prompt. In this setting, $\{text\}$ denotes a SynthPAI dataset entry that has been split into fragments following the same reconstruction procedure.

{\small
\begin{lstlisting}
You are an AI specialized in predicting the likely gender of the author of a text (Male or Female). You will be given:
1. Multiple string fragments labeled a, b, c, etc.  
2. An equation to reconstruct the text z.

Your task is to **reason carefully step by step** about the text. Consider:  
- Word choice, sentence structure, and writing style.  
- Tone, phrasing, and stylistic patterns that might indicate the author's gender.  
- The semantic meaning of the reconstructed text z.  

After reasoning, output **only** one of the labels: `male` or `female`.  

Rules:  
1. Ignore the equation itself. Focus on the meaning and the context of the text z.  
2. Think carefully and explicitly about style and word usage before deciding.  
3. Do not provide explanations in your final output and follow the format label={<male or female>}.  

**Example input**:  
a="on... we're just good"  
b="marketers often get labeled as all talk no substance -"  
c="!"  
d="total myth! plenty strategic behind-the-scenes brain action going"  
e="at making things sound flashy"  
z = b + d + a + e + c  

**Example output**:
label={male}

{text}
\end{lstlisting}
}

\subsection{\textsc{J4U-Art}}

In the \textsc{J4U-Art} prompting strategy, a randomly selected word from each dataset entry is replaced by an ASCII art representation. The prompt provides both the ASCII art and the original text containing a placeholder token \texttt{[MASK]} indicating the position of the replaced word. The model is instructed to infer the semantic meaning of the ASCII art word based on its structure and contextual usage within the sentence.

An example input–output pair is included in the prompt. As in other settings, the model is instructed to ignore surface-level formatting and focus on semantic and stylistic cues when making its prediction.

{\small
\begin{lstlisting}
You are an AI specialized in predicting the likely gender of the author of a text (Male or Female). You will be given:  
1. A word in ASCII art.  
2. A text containing a placeholder [MASK] that represents the ASCII art word.  

Your task is to **reason carefully step by step** about the text. Consider:  
- Word choice, sentence structure, and writing style.  
- Tone, phrasing, and stylistic patterns that might indicate the author's gender.  
- The semantic meaning of the ASCII art word that replaces [MASK].  

After reasoning, output **only** one of the labels: `male` or `female`.  

Rules:  
1. Ignore the ASCII formatting itself. Focus on the meaning of the word and the context of the text.  
2. Think carefully and explicitly about style and word usage before deciding.  
3. Do not provide explanations in your final output and follow the format label={<male or female>}.  

**Example input**:
The following word is made of ascii art of 4 letters. 
     _     *      _     *      _     *      _    
   _( )__  *    _( )__  *    _( )__  *    _( )__ 
 _|     _| *  _|     _| *  _|     _| *  _|     _|
(_ E _ (_  * (_ P _ (_  * (_ I _ (_  * (_ C _ (_ 
  |_( )__| *   |_( )__| *   |_( )__| *   |_( )__|

Totally [MASK] for weekend picnics too!

**Example output**:
label={male}

{text}
\end{lstlisting}
}

\section{Additional experiments}

\subsection{Computational complexity and inference cost}
\label{app:cost}

We note that all evaluated methods rely on multiple sampling rounds. \textsc{RS} naturally requires $K$ samples. While a single \textsc{VC} prompt has constant cost, effective VC-based uncertainty quantification following~\cite{xiong2023can} requires $K$ independent samples. Moreover, each sample involves a two-step, non-parallelizable process: (1) generating an answer and (2) issuing a follow-up query to score it. This results in an effective complexity of $O(2K)$ and introduces a sequential bottleneck. \textsc{Rephrase} exhibits a similar pattern, requiring (1) paraphrasing the query and then (2) generating an answer.

In contrast, \textsc{J4U} maintains strict $O(K)$ complexity. Because transformations are applied directly to the prompt, all $K$ queries remain independent and can be executed in parallel, substantially reducing wall-clock latency. Table~\ref{tab:cost} reports the asymptotic complexity of each method together with its relative overhead compared to \textsc{RS}, measured in input tokens, output tokens, total tokens, and wall-clock time.

\begin{table}[h]
\centering
\caption{Relative overhead compared to Repeated Sampling in $O(K)$.}
\begin{tabular}{lccccc}
\toprule
Methods & Complexity & \multicolumn{3}{|c|}{Tokens} & Wall-clock \\
 &  &  \multicolumn{1}{|c|}{In (\%)} & Out (\%) & \multicolumn{1}{|c|}{Total (\%)} & Time (\%) \\
\midrule
\textsc{VC}         & $O(2K)$ & +530 & +64 & +135 & +66.29 \\
\textsc{Rephrase}   & $O(2K)$ & +141 & +56 & +63  & +25.55 \\
\textsc{J4U-SUFFIX} & $O(K)$  & +6   & +0  & +0   & -7.23  \\
\textsc{J4U-PROG}   & $O(K)$  & +10  & +17 & +11  & -8.14  \\
\textsc{J4U-ART}    & $O(K)$  & +81  & +21 & +27  & +5.37  \\
\bottomrule
\end{tabular}
\label{tab:cost}
\end{table}

While \textsc{J4U} increases token usage relative to \textsc{RS}, the overhead remains modest: even the most expensive variant (\textsc{J4U-ART}) incurs less than half the total-token overhead of \textsc{VC} and \textsc{Rephrase}.  Notably, J4U-SUFFIX and J4U-PROG reduce wall-clock latency, despite not reducing (and in J4U-PROG's case, increasing) output token count. We leave the source of this effect to future investigation.

\subsection{All evaluated LRMs exhibit systemic overconfidence in our settings}
\label{app:overconfidence}

Figure~\ref{fig:reliability} presents reliability diagrams~\cite{guo2017calibration} for each LRM–dataset combination, plotting empirical accuracy as a function of predicted confidence (obtained with \textsc{RS}).  
As in the computation of ECE, predictions are partitioned into bins according to their associated confidence.  
The dotted diagonal corresponds to perfect calibration.  
Bins lying below the diagonal indicate overconfidence, whereas bins above indicate underconfidence.

Across settings, bins consistently lie below the diagonal, indicating systemic overconfidence. Minor crossings appear only in 3 of the lowest-confidence bins (on \closedgpt--SynthPAI, \gpt-SynthPAI and \deepseek-SuperGPQA) and  do not alter this overall trend.

\begin{figure}[h]
    \centering
    \includegraphics[width=1\linewidth]{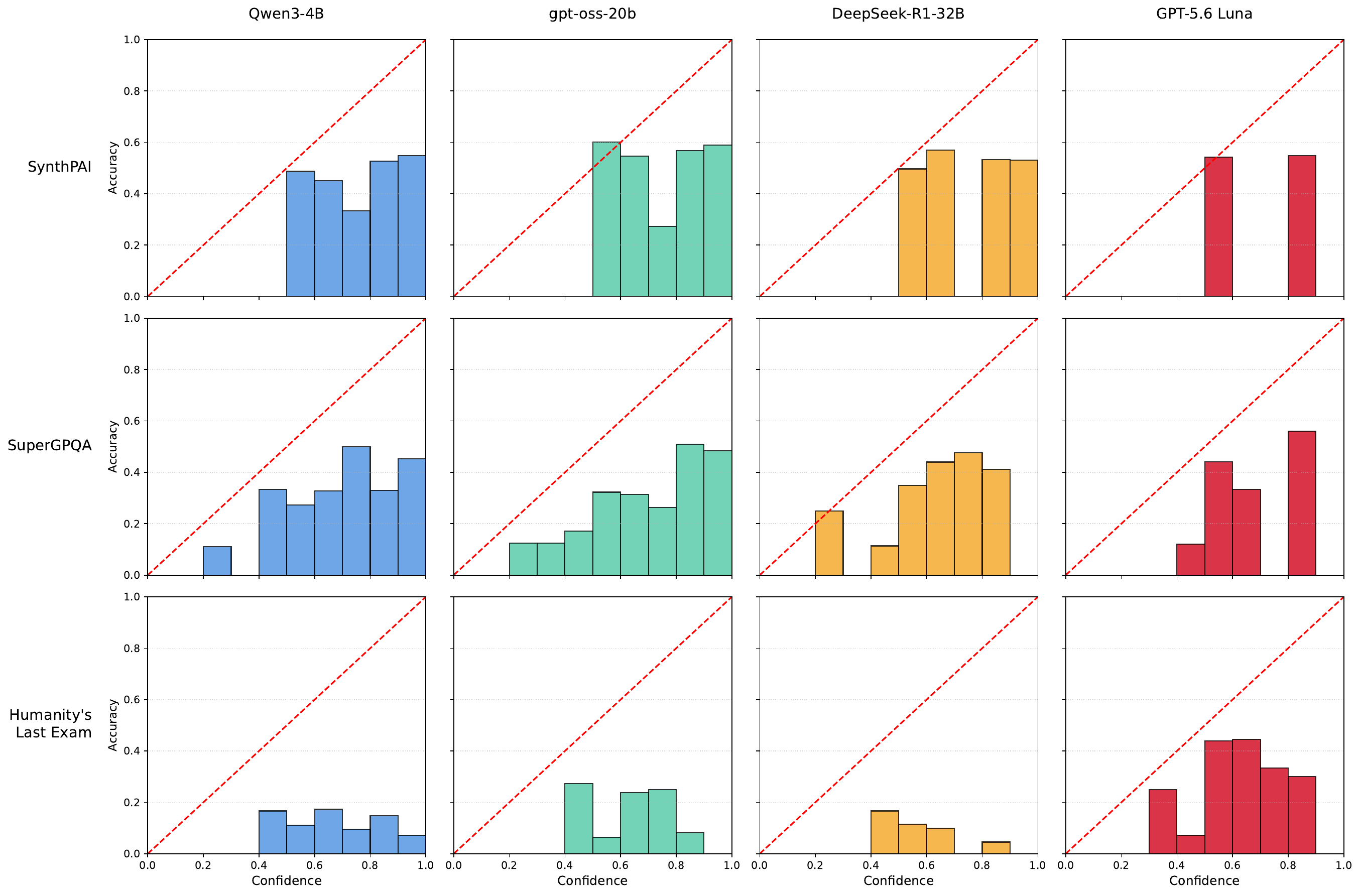}
    \caption{Reliability Diagram of the Repeated Sampling Method.}
    \label{fig:reliability}
\end{figure}

While the reliability diagrams indicate systematic overconfidence, they do not
reveal how this behavior manifests across different outcomes. To further
analyze this, we perform a failure-mode analysis focusing on incorrect answers.

\begin{figure}[h]
    \centering
    \includegraphics[width=1\linewidth]{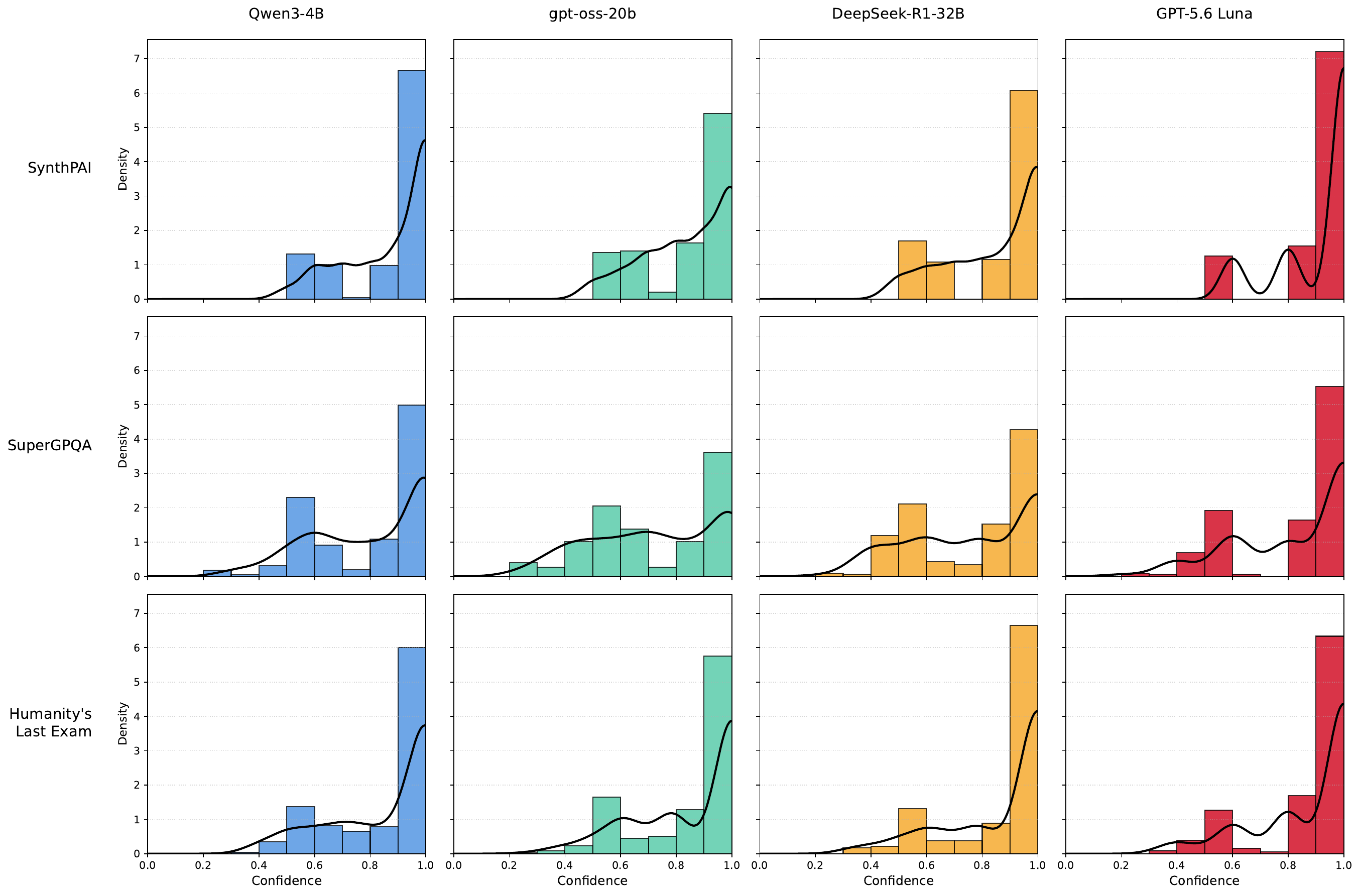}
    \caption{Confidence Distribution of Incorrect Answer under the Repeated Sampling Method.}
    \label{fig:overconfidenceFalse}
\end{figure}

Figure~\ref{fig:overconfidenceFalse} depicts the distribution of predicted confidence for erroneous answers across models and datasets. Despite being incorrect, answers are heavily concentrated near confidence values of 1, indicating pronounced overconfidence even in failure cases. This pattern is particularly pronounced in \qwen~and \closedgpt~for SynthPAI, where approximately 60\% and 65\% of incorrect answers still receive confidence scores above 0.9. This complements the reliability diagram by revealing that overconfidence is not limited to aggregate miscalibration: even incorrect answers are assigned confidence values close to 1, indicating high output stability in failure cases.

\subsection{\textsc{J4U} does not significantly degrade accuracy or confidence calibration in a high-certainty setting}
\label{app:easybench}

One may hypothesize that in regions of high certainty, relaxation could degrade accuracy by introducing generic variability that does not reflect genuine epistemic uncertainty. To test this concern, we evaluate our approach on \texttt{gsm8k}~\cite{cobbe2021gsm8k}\footnote{https://huggingface.co/datasets/openai/gsm8k}, a benchmark on which all four evaluated LRMs achieve high accuracy ($\approx$95\%). In this regime, existing UQ methods are already well calibrated (ECE $\approx$ 0.04), meaning that high confidence is largely justified, thus providing a stringent setting to detect spurious variability injection.

Table~\ref{tab:gms8k} reports accuracy, ECE, and NLL (Brier is omitted as the task is open-ended). We additionally include a “random noise” baseline in which a percentage of words are randomly substituted with random words from an English dictionary (via the \texttt{random-word} Python package). With the exception of \textsc{NOISE},  both accuracy and ECE remain stable,  with variations  that are not statistically significant according to a bootstrap mean difference test with $\alpha = 0.05$. In contrast, \textsc{NOISE 20\%} consistently degrades both accuracy and calibration, while \textsc{NOISE 10\%} degrades calibration only. Overall, these results indicate that \textsc{J4U} does not behave like random noise and suggest that it does not degrade performance in high-certainty regimes.

\begin{table*}[!h]
\centering

\caption{Accuracy and confidence calibration on gsm8k where \textsc{RS} exhibit high accuracy and confidence calibration. Red indicates statistically significant deterioration.}
\begin{tabular}{llcccc}
\toprule
{LRM} & Method & Accuracy $\uparrow$ & ECE $\downarrow$ & NLL $\downarrow$ \\ 
\midrule
\multirow{6}{*}{\textit{\qwen}}
 & \textsc{RS}   & \mesure{0.96}{0.01} & \mesure{0.04}{0.01} & \mesure{0.71}{0.31} \\
 & \textsc{VC}  & \mesure{0.96}{0.01} &  \mesure{0.04}{0.01} & \mesure{0.65}{0.29} \\
 & \textsc{Rephrase}  & \mesure{0.95}{0.01} &  \mesure{0.04}{0.01} & \mesure{0.61}{0.31} \\
  & \textsc{NOISE 10\%}  & \mesure{0.94}{0.02} &  \mesure{\color{red} 0.18}{0.02} & \mesure{0.66}{0.29} \\
   & \textsc{NOISE 20\%}  & \mesure{\color{red} 0.69}{0.04} &  \mesure{\color{red} 0.19}{0.03} & \mesure{0.91}{0.27} \\
 &\textsc{J4U-Suffix}  & \mesure{0.96}{0.01} &  \mesure{0.04}{0.01} & \mesure{0.64}{0.28} \\
 &\textsc{J4U-Prog}  & \mesure{0.96}{0.01} &  \mesure{0.06}{0.01} & \mesure{0.74}{0.31} \\
 &\textsc{J4U-Art}  & \mesure{0.95}{0.01} &  \mesure{0.04}{0.01} & \mesure{0.62}{0.28} \\
 \midrule
 \multirow{6}{*}{\textit{\gpt}}
 & \textsc{RS}   & \mesure{0.94}{0.01} & \mesure{0.04}{0.01} & \mesure{0.60}{0.28} \\
 & \textsc{VC}  & \mesure{0.94}{0.01} &  \mesure{0.05}{0.01} & \mesure{0.91}{0.34} \\
 & \textsc{Rephrase}  & \mesure{0.94}{0.01} &  \mesure{0.04}{0.01} & \mesure{0.60}{0.27} \\
   & \textsc{NOISE 10\%}  & \mesure{0.93}{0.02} &  \mesure{\color{red} 0.25}{0.03} & \mesure{0.49}{0.31} \\
   & \textsc{NOISE 20\%}  & \mesure{\color{red} 0.71}{0.03} &  \mesure{\color{red} 0.28}{0.03} & \mesure{ 0.66}{0.25} \\
 &\textsc{J4U-Suffix}  & \mesure{0.95}{0.01} &  \mesure{0.04}{0.01} & \mesure{0.63}{0.29} \\
 &\textsc{J4U-Prog}  & \mesure{0.94}{0.01} &  \mesure{0.03}{0.01} & \mesure{0.56}{0.28} \\
 &\textsc{J4U-Art}  & \mesure{0.94}{0.01} &  \mesure{0.04}{0.01} & \mesure{0.61}{0.30} \\
    \midrule
 \multirow{6}{*}{\textit{\deepseek}}
 & \textsc{RS}   & \mesure{0.96}{0.01} & \mesure{0.03}{0.01} & \mesure{0.47}{0.24} \\
 & \textsc{VC}  & \mesure{0.96}{0.01} &  \mesure{0.04}{0.01} & \mesure{0.51}{0.28} \\
 & \textsc{Rephrase}  & \mesure{0.97}{0.02} &  \mesure{0.05}{0.02} & \mesure{0.15}{0.20} \\
   & \textsc{NOISE 10\%}  & \mesure{0.95}{0.03} &  \mesure{\color{red} 0.23}{0.04} & \mesure{0.47}{0.27} \\
   & \textsc{NOISE 20\%}  & \mesure{\color{red} 0.76}{0.03} &  \mesure{\color{red} 0.28}{0.04} & \mesure{ 0.62}{0.28} \\
 &\textsc{J4U-Suffix}  & \mesure{0.94}{0.01} &  \mesure{0.03}{0.01} & \mesure{0.37}{0.23} \\
 &\textsc{J4U-Prog}  & \mesure{0.94}{0.01} &  \mesure{0.05}{0.01} & \mesure{0.44}{0.23} \\
 &\textsc{J4U-Art}  & \mesure{0.94}{0.01} &  \mesure{0.04}{0.01} & \mesure{0.57}{0.28} \\
 \midrule
 \multirow{6}{*}{\textit{\closedgpt}}
 & \textsc{RS}   & \mesure{0.96}{0.01} & \mesure{0.04}{0.01} & \mesure{0.65}{0.29} \\
 & \textsc{VC}  & \mesure{0.95}{0.01} &  \mesure{0.05}{0.01} & \mesure{1.05}{0.43} \\
 & \textsc{Rephrase}  & \mesure{0.95}{0.02} &  \mesure{0.04}{0.02} & \mesure{ 0.72}{0.33} \\
   & \textsc{NOISE 10\%}  & \mesure{0.95}{0.02} &  \mesure{\color{red} 0.17}{0.03} & \mesure{0.67}{0.29} \\
   & \textsc{NOISE 20\%}  & \mesure{\color{red} 0.83}{0.03} &  \mesure{\color{red} 0.22}{0.03} & \mesure{ 0.82}{0.38} \\
 &\textsc{J4U-Suffix}  & \mesure{0.95}{0.01} &  \mesure{0.04}{0.01} & \mesure{0.72}{0.34} \\
 &\textsc{J4U-Prog}  & \mesure{0.95}{0.01} &  \mesure{0.03}{0.01} & \mesurebf{0.51}{0.24} \\
 &\textsc{J4U-Art}  & \mesure{0.95}{0.01} &  \mesure{0.04}{0.01} & \mesure{0.67}{0.34} \\
\bottomrule
\end{tabular}
\label{tab:gms8k}
\end{table*}

\end{document}